\documentclass{article}
\usepackage{ijcai22}

\usepackage{times}
\usepackage{soul}
\usepackage{url}
\usepackage[hidelinks]{hyperref}
\usepackage[utf8]{inputenc}
\usepackage[small]{caption}
\usepackage{graphicx}

\usepackage{latexsym}
\usepackage{danudefs}
\usepackage{algorithmic}
\usepackage{algorithm}
\usepackage{color}
\usepackage{booktabs}
\usepackage{xspace}

\usepackage{caption}
\usepackage{subcaption}
\usepackage{multirow}
\usepackage{enumerate}
\usepackage{esvect}
\usepackage{csvsimple}

\usepackage{microtype}

\usepackage[english]{babel}
\usepackage{hyperref}
\addto\captionsenglish{%
}
\addto\extrasenglish{%
}

\usepackage{array}
\newcolumntype{H}{>{\setbox0=\hbox\bgroup}c<{\egroup}@{}}

\newtheorem{lemma}[]{Lemma}

\newcommand{\PIP}[1]{\mathrm{PIP}(#1)}

\newcommand{\newcite}[1]{\citeauthor{#1}~[\citeyear{#1}]}

\title{Learning Meta Word Embeddings by Unsupervised Weighted \\Concatenation of Source Embeddings}

\author{
Danushka Bollegala
\affiliations
University of Liverpool, Amazon
\emails
\texttt{danushka@liverpool.ac.uk}
}

\date{}

\begin{document}
\maketitle

\begin{abstract}
	Given multiple source word embeddings learnt using diverse algorithms and lexical resources, meta word embedding learning methods attempt to learn more accurate and wide-coverage word embeddings. 
	Prior work on meta-embedding has repeatedly discovered that simple vector concatenation of the source embeddings to be a competitive baseline. 
	However, it remains unclear as to why and when simple vector concatenation can produce accurate meta-embeddings. 
	We show that weighted concatenation can be seen as a spectrum matching operation between each source embedding and the meta-embedding, minimising the pairwise inner-product loss.
	Following this theoretical analysis, we propose two \emph{unsupervised} methods to learn the optimal concatenation weights for creating meta-embeddings from a given set of source embeddings.
	Experimental results on multiple benchmark datasets show that the proposed weighted concatenated meta-embedding methods outperform previously proposed meta-embedding learning methods.
\end{abstract}

\section{Introduction}
\label{sec:intro}

Distributed word representations have shown impressive performances in multiple, diverse, downstream NLP applications when used as features~\cite{Huang:ACL:2012,Milkov:2013,Pennington:EMNLP:2014}. 
The learning objectives, optimisation methods as well as the lexical resources used in these word embedding learning methods vary significantly, resulting in a diverse set of word embeddings that capture different aspects of lexical semantics such as polysemy~\cite{neelakantan-EtAl:2014:EMNLP2014,Reisinger:NAACL:2010,Huang:ACL:2012} and morphology~\cite{cotterell-schutze-eisner:2016:P16-1}.
Meta-embedding has emerged as a promising solution to combine diverse pre-trained \emph{source} word embeddings for the purpose of producing \emph{meta} word embeddings that preserve the complementary strengths of individual source word embeddings.
The input and output word embeddings to the meta-embedding algorithm are referred respectively as the source and meta-embeddings.

The problem setting of meta-embedding learning differs from that of source word embedding learning in several important aspects.
In a typical word embedding learning scenario, we would randomly initialise the word embeddings and subsequently update them such that some goodness criterion is optimised such as predicting the log co-occurrences in Global Vectors (\textbf{GloVe})~\cite{Pennington:EMNLP:2014} or likelihood in skip-gram with negative sampling (\textbf{SGNS})~\cite{Milkov:2013}.
Therefore, the source word embedding methods can significantly differ in their training objectives and optimisation methods being used.
On the other hand, for a meta-embedding learning method to be generalisable to different source embedding learning methods, it must be agnostic to the internal mechanisms of  the source embedding learning methods.
Moreover, the vector spaces as well as their optimal dimensionalities will be different for different source embeddings, which makes it difficult to directly compare source embeddings.

Despite the above-mentioned challenges encountered in meta-embedding learning, it has several interesting properties.
From a practitioners point-of-view, meta-embedding is attractive because it obviates the need to pick a single word embedding learning algorithm, which can be difficult because different word embedding learning algorithms perform best in different downstream NLP tasks under different settings~\cite{Levy:TACL:2015}.
Moreover, meta-embedding learning does not require the original linguistic resources (which might not be publicly available due to copyright issues) used to learn the source word embeddings, and operates directly on the (often publicly available) pre-trained word embeddings.
Even in cases where the original linguistic resources are available, retraining source word embeddings from scratch can be time consuming and require specialised hardware.

Given a set of source embeddings, a simple yet competitive baseline for creating their meta-embedding is to concatenate the source embeddings~\cite{Bollegala:IJCAI:2018,Yin:ACL:2016,AAAI:2016:Goikoetxea}.
Concatenation has been justified in prior work as a method that preserves the information contained in individual sources in their corresponding vector spaces.
However, this explanation has no theoretical justification and it is unclear how to concatenate source embeddings with different accuracies and dimensionalities, or what losses are being minimised.  

\paragraph{Contributions:}
For word embedding methods that can be expressed as matrix factorisations, we show that their concatenated meta-embedding minimises the \emph{Pairwise Inner Product} (\textbf{PIP})~\cite{Yin:2018} loss between the sources and an ideal meta-embedding.
Specifically, we show that PIP loss can be decomposed into a \emph{bias} term that evaluates the amount of information lost due to meta-embedding and a series of \emph{variance} terms that account for how source embedding spaces should be aligned with the ideal meta-embedding space to minimise the PIP loss due to meta-embedding.
Our theoretical analysis extends the bias-variance decomposition of PIP loss for word embedding~\cite{Yin:2018} to meta-embedding.

Motivated by the theory, we propose two \emph{unsupervised} methods to learn the optimal concatenation weights by aligning the spectrums of the source embeddings against that of an (estimated) ideal meta-embedding matrix.
In particular, no labelled data for downstream tasks are required for learning the optimal concatenation weights.
We propose both source-weighted  and dimension-weighted concatenation methods.
In particular, the dimension-weighted meta-embedding learning method consistently outperforms prior proposals in a range of downstream NLP tasks
such as word similarity prediction, analogy detection,  part-of-speech tagging, sentiment classification, sentence entailment and semantic textual similarity prediction for various combinations of source embeddings.
The source code for reproducing the results reported in this paper is publicly available.\footnote{\url{https://github.com/LivNLP/meta-concat}}

\section{Related Work}
\label{sec:related}

\newcite{Yin:ACL:2016} proposed 1\texttt{TO}N,
 by projecting source embeddings to a common space via source-specific linear transformations.
This method minimises squared $\ell_{2}$ distance between the meta and each source embedding assuming a common vocabulary. 
1\texttt{TO}N+ overcomes this limitation by learning pairwise linear transformations between two given sources for predicting the embeddings for out of vocabulary (OOV) words.
Both of these methods can be seen as \emph{globally-linear} transformations because \emph{all} the words in a particular source are projected to the meta-embedding space using the \emph{same} transformation matrix.
In contrast, \emph{locally-linear} meta-embedding (LLE)~\cite{Bollegala:IJCAI:2018} computes the nearest neighbours for a particular word in each source and then represent a word as the linearly-weighted combination of its neighbours.
Next, meta-embeddings are learnt by preserving the neighbourhood constraints.
This method does not require all words to be represented by all sources, thereby obviating the need to predict missing source embeddings for OOVs.

\newcite{Bao:COLING:2018} modelled meta-embedding learning as an \emph{autoencoding} problem where information embedded in different sources are integrated at different levels to propose three types of meta-embeddings: Decoupled Autoencoded Meta-Embedding (DAEME) (independently encode each source and  concatenate), Concatenated Autoencoded Meta-Embedding (CAEME) (independently decode meta-embeddings to reconstruct each source), and Averaged Autoencoded Meta-Embedding (AAEME) (similar to DAEME but instead of concatenation use average).
In comparison to methods that learn globally or locally linear transformations~\cite{Bollegala:IJCAI:2018,Yin:ACL:2016}, autoencoders learn nonlinear transformations.
\newcite{Neill:2018} further extend this approach using squared cosine proximity loss as the reconstruction loss. 

Vector concatenation has been used as a baseline for producing 
meta-embeddings~\cite{Yin:ACL:2016}.
\newcite{AAAI:2016:Goikoetxea} concatenated independently learnt word embeddings from a corpus and the WordNet.
Moreover, applying Principal Component Analysis on the concatenation further improved their performance on similarity tasks.
Interestingly, \newcite{Coates:NAACL:2018} showed that accurate meta-embeddings can be produced by averaging source embeddings that exist in \emph{different} vector spaces. 
Recent work~\cite{He:2020,jawanpuria-etal-2020-learning} show that learning orthogonal transformations prior to averaging can further improve accuracy.

\newcite{kiela-etal-2018-dynamic} proposed a dynamically weighted meta-embedding method for representing sentences by projecting each source by a source-specific matrix to a common vector space such that the projected source embeddings can be added after multiplying an attention weight. 
They consider contextualised word embeddings given by the concatenated forward and backward hidden states of a BiLSTM. 
The attention weights are learnt using labelled data for sentence-level tasks (i.e. sentiment classification and textual entailment). 
However, their method differs from all previously discussed meta-embedding learning methods in two important ways:
 it (a) meta-embeds contextualised word embeddings and 
(b) requires labelled data for sentence-level tasks and is more appropriate for creating sentence-level supervised meta-embeddings than word-level unsupervised meta-embeddings.

Both contextualised embeddings and sentence-level meta-embeddings are beyond the scope of this paper, which focuses on context-independent (static) word-level meta-embedding.
Although theoretical analysis of static word embedding learning exists~\cite{Arora:TACL:2016,li-etal-2015-generative,Hashimoto:TACL:2016}, to the best of our knowledge, we are the first to provide a theoretical analysis of concatenated meta-embedding learning.

\section{Meta-Embedding by Concatenation}
\label{sec:wc}

Let us consider a set of $N$ source word embeddings $s_{1}, s_{2}, \ldots, s_{N}$,
 all covering a common vocabulary\footnote{Missing embeddings can be predicted using, for example, linear projections as in 1\texttt{TO}N+.} $\cV$ of $n$ words (i.e. $|\cV| = n$).
The embedding of a word $w$ in $s_{j}$ is denoted by $\vec{s}_{j}(w) \in \R^{k_{j}}$, where $k_{j}$ is the dimensionality of $s_{j}$.
We represent $s_{j}$ by an embedding matrix $\mat{E}_{j} \in \R^{n \times k_{j}}$.
For example, $\mat{E}_{1}$ could be the embedding matrix obtained by running SGNS on a corpus, whereas $\mat{E}_{2}$ could be that obtained from GloVe etc.
Then, the problem of meta-embedding via weighted concatenation can be stated as -- \emph{what are the optimal weights to concatenate $\mat{E}_{1}, \ldots, \mat{E}_{n}$ row-wise such that some loss that represents the amount of information we lose due to meta-embedding is minimised?}

\subsection{Two Observations}
\label{sec:observations}

We build our analysis on two main observations. 
\textbf{First}, we note that word embeddings to be unitary-invariant. 
Unitary invariance is empirically verified in prior work~\cite{artetxe-labaka-agirre:2016:EMNLP2016,hamilton-leskovec-jurafsky:2016:P16-1,Smith:ICLR:2017} and states that two source embeddings are identical if one can be transformed into the other by a unitary matrix. 
Unfortunately the dimensions in different source embeddings cannot be directly compared~\cite{Bollegala:PLoS:2017}.
To overcome this problem, \newcite{Yin:2018} proposed the Pairwise Inner-Product (PIP) matrix, $\mathrm{PIP}(\mat{E})$, of $\mat{E}$ given by \eqref{eq:pip} to compare source embeddings via their inner-products over the word embeddings for the same set of words. 
\begin{align}
\label{eq:pip}
\PIP{\mat{E}} = \mat{E}\mat{E}\T
\end{align}
If the word embeddings are normalised to unit $\ell_{2}$ length, $\mat{E}_{j}\mat{E}_{j}\T$ becomes the pairwise word similarity matrix.
PIP loss between two embedding matrices $\mat{E}_{1}, \mat{E}_{2}$ is defined as the Frobenius norm of the difference between their PIP matrices:
\par\nobreak
{\small
\vspace{-5mm}
\begin{align}
 \label{eq:pip-loss}
 \norm{\PIP{\mat{E}_{1}} \! - \! \PIP{\mat{E}_{2}}}_{F} = \norm{\mat{E}_{1}\mat{E}_{1}\!\!\T  \! - \! \mat{E}_{2}\mat{E}_{2}\!\!\T}_{F}
\end{align}
}%
PIP loss enables us to compare word emebddings with different dimensionalities, trained using different algorithms and resources, which is an attractive property when analysing meta embeddings. 

\noindent
\textbf{Second}, we observe that many word embedding learning algorithms such as Latent Semantic Analysis (LSA)~\cite{LSA}, GloVe, SGNS, etc. can be written as either an explicit or an implicit low-rank decomposition of a suitably transformed co-occurrence matrix, computed from a corpus~\cite{li-etal-2015-generative,Dhillon:2015}. 
For example, LSA applies Singular Value Decomposition (SVD) on a Positive Pointwise Mutual Information (PPMI) matrix, GloVe decomposes a log co-occurrence matrix, and SGNS implicitly decomposes a Shifted PMI (SPMI) matrix.

More generally, if $\mat{M}$ is a signal matrix that encodes information about word associations (e.g. a PPMI matrix) and $\mat{M} = \mat{U}\mat{D}\mat{V}\T$ be its  SVD, then a $k$-dimensional embedding matrix $\mat{E}$ of $\mat{M}$ is given by $\mat{E} = \mat{U}_{1:k}\mat{D}^{\alpha}_{1:k,1:k}$ for some $\alpha \in [0,1]$ by selecting the largest $k$ singular values (in the diagonal matrix $\mat{D}$) and corresponding left singular vectors (in the unitary matrix $\mat{U}$)~\cite{Turney:JAIR:2012}.
Setting $\alpha = 0.5$ induces symmetric target and context vectors, similar to those obtained via SGNS or GloVe and has been found to be empirically more accurate~\cite{Levy:NIPS:2014}. 
The hyperparameter $\alpha$ was found to be controlling the robustness of the embeddings against over-fitting~ \cite{Yin:2018}~(See  Supplementary for a discussion of $\alpha$).
Given, $\alpha$ and $k$, a source word embedding learning algorithm can then be expressed as a function 
$ \mat{E} = f_{\alpha, k}(\mat{M}) =  \mat{U}_{1:k}\mat{D}^{\alpha}_{1:k,1:k}$ that returns an embedding matrix $\mat{E}$ for an input signal matrix $\mat{M}$.

Having stated the two observations we use to build our analysis, next we propose two different approaches for constructing meta-embeddings as the weighted concatenation of source embeddings.

\subsection{Source-weighted Concatenation}
\label{sec:ss}

\newcite{Yin:ACL:2016} observed that it is important to emphasise accurate source embeddings during concatenation by multiplying all embeddings of a particular source by a source-specific weight, which they tune using a semantic similarity benchmark.
Specifically, they compute the \emph{source-weighted} meta-embedding, $\hat{\vec{e}}_{sw}(w) \in \R^{k_{1} + \cdots + k_{N}}$, of a word $w \in \cV$ using \eqref{eq:ss}, where $\oplus$ denotes the vector concatenation.
\begin{align}
\label{eq:ss}
\hat{\vec{e}}_{sw}(w) &= c_{1}\vec{s}_{1}(w)\oplus  \ldots \oplus c_{N} \vec{s}_{N}(w) \nonumber \\
& = \oplus_{j=1}^{N} c_{j}\vec{s}_{j}(w)
\end{align}
Concatenation weights $c_{j}$ satisfy $\sum_{j=1}^{N} c_{j} = 1$.
Then, the source-weighted concatenated meta-embedding matrix, $\hat{\mat{E}}_{sw}$ is given by \eqref{eq:ss-M}.
\begin{align}
\label{eq:ss-M}
\hat{\mat{E}}_{sw} = \bigoplus_{j=1}^{N} c_{j} \mat{E}_{j},
\end{align}
 $\bigoplus$ to denotes the row-wise matrix concatenation.

\subsection{Dimension-wighted Concatenation}
\label{sec:ds}

The number of source embeddings is usually significantly smaller compared to their sum of dimensionalities (i.e. $N\! \ll \! \sum_{j=1}^{N} k_{j}$).
Moreover, the source-weighted concatenation can only adjust the length of each source embedding, and cannot perform any rotations.
Therefore, the flexibility of the source-weighting to produce accurate meta-embeddings is limited.
To overcome these limitations, we propose a \emph{dimension-weighted} concatenation method given by \eqref{eq:ds}.
\begin{align}
	\label{eq:ds}
	\hat{\vec{e}}_{dw}(w) = \oplus_{j=1}^{N} \mat{C}_{j}\vec{s}_{j}(w)
\end{align}
Here, $\mat{C}_{j}$ is a diagonal matrix with $c_{j,1}, \ldots, c_{j,k_{j}}$ in the main diagonal.
We require that for all $j$, $\sum_{i=1}^{k_{j}} c_{j,i} = 1$.
The dimension-weighted concatenated meta-embedding matrix $\hat{\mat{E}}_{dw}$ can be written as follows:
\par\nobreak
{\small
\vspace{-6mm}
\begin{align}
\label{eq:ds-M}
	\hat{\mat{E}}_{dw} = \bigoplus_{j=1}^{N} \mat{E}_{j} \mat{C}_{j} 
\end{align}
}%
Here, we have $\sum_{j=1}^{N }k_{j} (\gg\!N)$ number of parameters at our disposal to scale each dimension of the sources, which is more flexible than the source-weighted concatenation.
Indeed, source-weighting can be seen as a special case of dimension-weighting when
$c_{j} = c_{j,1} = \ldots c_{j,k_{j}}$ for all $j$.

\subsection{Bias-Variance in Meta-Embedding}

Armed with the two key observations in \autoref{sec:observations}, we are now in a position to show how meta-embeddings under source- and dimension weighted concatenation induce a bias-variance tradeoff in the PIP loss. 
Given that source-weighting is a special case of dimension-weighting, we limit our discussion to the latter.
Moreover, for simplicity of the description we consider two sources here but it can be extended any number of sources.\footnote{Indeed we use three sources later in experiments.}

Let us consider the dimension-weighted concatenated meta-embedding $\hat{\mat{E}} = [\mat{E}_{1}\mat{C}_{1}; \mat{E}_{2}\mat{C}_{2}]$ of two source embedding matrices $\mat{E}_{1} \in \R^{n \times k_{1}}$ and $\mat{E}_{2} \in \R^{n \times k_{2}}$ with concatenation coefficient matrices 
$\mat{C}_{1} = \mathrm{diag}(c_{1,1}, \ldots, c_{1,k_{1}})$ and $\mat{C}_{2} = \mathrm{diag}(c_{2,1}, \ldots, c_{2,k_{2}})$.
$\mat{E}_{1}$ and $\mat{E}_{2}$ are obtained by applying SVD on respective signal matrices $\mat{M}_{1}$ and $\mat{M}_{2}$ and are given by
$\mat{E}_{1} = \mat{U}^{(1)}_{\cdot, 1:{k}_{1}} \mat{D}^{(1)\alpha}_{1:{k}_{1}, 1:{k}_{1}}$ and
$\mat{E}_{2} = \mat{U}^{(2)}_{\cdot, 1:{k}_{2}}\mat{D}^{(2)\alpha}_{1:{k}_{2},1:{k}_{2}}$.
The diagonal matrices $\mat{D}^{(1)}_{1:{k}_{1}, 1:{k}_{1}} = \mathrm{diag}(\mu_{1}, \ldots, \mu_{{k}_{1}})$ and $\mat{D}^{(2)}_{1:{k}_{2}, 1:{k}_{2}} = \mathrm{diag}(\nu_{1}, \ldots, \nu_{{k}_{2}})$ contain the top $k_{1}$ and $k_{2}$ singular values of respectively of $\mat{M}_{1}$ and $\mat{M}_{2}$.

Let us assume the exsistence of an oracle that provides us with an \emph{ideal} meta-embedding matrix $\mat{E}  \in \R^{n \times d}$, where $d \geq (k_{1} + k_{2})$ is the dimensionality of this ideal meta-embedding space. 
$\mat{E}$ is ideal in the sense that it has the minimal PIP loss $\norm{\PIP{\mat{E}} - \PIP{\hat{\mat{E}}}}_{F}$ between $\hat{\mat{E}}$ created from source embedding matrices $\mat{E}_{1}$ and $\mat{E}_{2}$.
Following the low-rank matrix decomposition approach, $\mat{E} =  \mat{U}_{\cdot, 1:d}\mat{D}^{\alpha}_{1:d,1:d}$ can be computed using an ideal signal matrix $\mat{M}$, where $\mat{D}_{1:d,1:d} = \mathrm{diag}(\lambda_{1}, \ldots, \lambda_{d})$.
However, in practice, we are not blessed with such oracles and will have to estimate $\mat{M}$ via sampling as described later in \autoref{sec:est}.
Interestingly, in the case of $\hat{\mat{E}}$ constructed by dimension-weighted concatenation, we can derive an upper bound on the PIP loss using the spectral decompositions of $\mat{M}, \mat{M}_{1}$ and $\mat{M}_{2}$ as stated in \autoref{th:main}. 
\begin{theorem}
\label{th:main}
For two source embedding matrices $\mat{E}_{1}$ and $\mat{E}_{2}$, the PIP loss between their dimension-weighted meta-embedding $\hat{\mat{E}} = [\mat{E}_{1}\mat{C}_{1}; \mat{E}_{2}\mat{C}_{2}]$ and an ideal meta-embedding $\mat{E}$ is given by \eqref{eq:PIP-ME}.
\par\nobreak
{\small
\vspace{-3mm}
\begin{align}
\label{eq:PIP-ME}
& \norm{\PIP{\mat{E}} - \PIP{\hat{\mat{E}}}}_{F} \nonumber \\
& \leq \underbrace{\sqrt{\sum_{i=k_{1} + k_{2}}^{d} \lambda_{i}^{4\alpha}}}_{\textrm{bias}} +
 \textcolor{red}{\underbrace{\sqrt{2}\sum_{i=1}^{k_{1}}\left(\lambda_{i}^{2\alpha} - \lambda_{i+1}^{2\alpha}\right)\norm{\mat{U}^{(1)}_{\cdot,1:i}\T \mat{U}_{\cdot,i:n}}}_{\textrm{directional variance in $s_{1}$}}} \nonumber \\
&\!\!+\! \textcolor{red}{\underbrace{\sqrt{\sum_{i=1}^{k_{1}}\left(\lambda_{i}^{2\alpha} \!-\! c^{2}_{1,i}\mu_{i}^{2\alpha} \right)^{2}}}_{\textrm{magnitude variance in $s_{1}$}}} \!+\! 
\textcolor{blue}{\underbrace{\sqrt{\sum_{i=k_{1}+1}^{k_{1}+k_{2}}\left(\lambda_{i}^{2\alpha} \!-\! c^{2}_{2,i-k_{1}} \nu_{i-k_{1}}^{2\alpha}\right)^{2}}}_{\textrm{magnitude variance in $s_{2}$}}} \nonumber \\
&+ \textcolor{blue}{\underbrace{\sqrt{2}\sum_{i=k_{1} + 1}^{k_{1} + k_{2}} \left(\lambda_{i}^{2\alpha} - \lambda_{i+1}^{2\alpha} \right) \norm{\mat{U}^{(2)}_{\cdot,1:i}\T \mat{U}_{\cdot,i:n}}}_{\textrm{directional variance in $s_{2}$}}}
\end{align}
}
\end{theorem}
\begin{proof}
See supplementary.
\end{proof}

For symmetric signal matrices such as the ones computed via word co-occurrences, embedding matrices are given by the eigendecomposition of the signal matrices and with real eigenvalues.
$\mat{M}, \mat{M}_{1}, \mat{M}_{2}$ (hence their spectrums) as well as respective optimal dimensionalities $d$, $k_{1}$, $k_{2}$ are unknown but can be estimated  as described in \autoref{sec:est}.
Although we used the same $\alpha$ for the oracle and all sources for simplicity, Theorem~1 still holds when all $\alpha$ are different.

The perturbation bounds on the expected PIP losses are known to be tight~\cite{Vu:2011}.
The first term in the R.H.S. in \eqref{eq:PIP-ME} can be identified as a \emph{bias} due to meta-embedding because as we use more dimensions in the sources (i.e. $k_{1}+k_{2}$) for constructing a meta-embedding this term would become smaller, analogous to increasing the complexity of a prediction model. 
On the other hand, third and fourth terms in the R.H.S. are square roots of the sum of squared differences between the spectra of the ideal meta-embedding and each source embedding, weighted by the concatenation coefficients. 
These terms capture the variance in the magnitudes of the spectra.
Likewise, the second and fifth terms compare the left singular vectors (directions) between each source and the ideal meta-embedding, representing the variance in the directions.

The bias-variance decomposition \eqref{eq:PIP-ME} provides a principled approach for constructing concatenated meta-embeddings.
To minimise the PIP loss, we must \emph{match} the spectra of the sources against the ideal meta-embedding by finding suitable concatenation coefficients such that the magnitude variance terms are minimised. 
Recall that in meta-embedding, the sources are pretrained and given, thus their spectra are fixed.
Therefore, PIP loss \eqref{eq:PIP-ME} is a function only of concatenation coefficients 
for all $c_{1,i}, c_{2,j}$ and their optimal values can be obtained by differentiating w.r.t. those variables.
For example, for the first source, under dimension- and source-weighted concatenations the optimal weights are given respectively by \eqref{eq:sw-1} and \eqref{eq:dw-1}.
\par\nobreak
{\small
\noindent\begin{minipage}{.45\linewidth}
\begin{align} 
c_{1,i} = \frac{\lambda_{i}^{\alpha}}{\mu_{i}^{\alpha}}  \label{eq:sw-1}
\end{align}
\end{minipage}
\begin{minipage}{.45\linewidth}
\begin{align}
c_{1} = \sqrt{\frac{\sum_{i}^{{k}_{1}} \lambda_{i}^{2\alpha}\mu_{i}^{2\alpha}}{\sum_{i=1}^{{k}_{1}} \mu_{i}^{4\alpha}}} \label{eq:dw-1}
\end{align}
\end{minipage}
}

Interestingly, directional variances can be further minimised by, for example, rotating the sources to have orthonormal directions with the ideal meta-embedding.
Indeed, prior work~\cite{He:2020,jawanpuria-etal-2020-learning} have experimentally shown that by applying orthogonal transformations to the source embeddings we can further improve the performance in averaged meta-embeddings.
Alternatively, if we have a choice on what sources to use, we can use directional variance terms to select a subset of source embeddings for creating meta-embeddings.

\subsection{Ideal Embedding Estimation}
\label{sec:est}

In meta-embedding learning, we assume that we are given a set of trained source embeddings with specific dimensionalities.
Therefore, $k_{1}, \ldots, k_{N}, \alpha$ that determine the source embeddings, $\mat{E}_{1}, \ldots, \mat{E}_{N}$, obtained by applying SVD on the respective signal matrices, $\mat{M}_{1}, \ldots, \mat{M}_{N}$, are \emph{not} parameters in the meta-embedding learning task, but are related to the individual source embedding learning tasks.
On the other hand, the singular values, $\lambda_{1}, \ldots, \lambda_{d}$, of the ideal meta-embedding signal matrix, $\mat{M}$, its dimensionality, $d$, are unknown parameters related to the meta-embedding learning task and must be estimated.
In the case of concatenated meta-embeddings discussed in this paper, we have $d = \sum_{j} k_{j}$.
However, we must still estimate $\lambda_{1}, \ldots, \lambda_{d}$ as they are required in the computations of source- and dimension-weighted concatenation coefficients, given respectively by \eqref{eq:sw-1} and \eqref{eq:dw-1}.

For this purpose, we model a signal matrix, $\tilde{\mat{M}}$, that we compute from a corpus using some co-occurrence measure as the addition of zero mean Gaussian noise matrix $\mat{Z}$ to an ideal signal matrix $\mat{M}$, (i.e. $\tilde{\mat{M}} = \mat{M} + \mat{Z}$).
Let the spectrum of $\tilde{\mat{M}}$ be $\tilde{\lambda}_{1}, \ldots, \tilde{\lambda}_{d}$.
Next, we randomly split the training corpus into two equally large subsets, and compute two signal matrices: $\tilde{\mat{M}}_{1} = \mat{M} + \mat{Z}_{1}$ and $\tilde{\mat{M}}_{2} = \mat{M} + \mat{Z}_{2}$, where $\mat{Z}_{1}, \mat{Z}_{2}$ are two independent copies of noise with variance $2\sigma^{2}$.
Observing that $\tilde{\mat{M}}_{1} - \tilde{\mat{M}}_{2} = \mat{Z}_{1} - \mat{Z}_{2}$ is a random matrix with zero mean and $4\sigma^{2}$ variance, we can estimate the noise for symmetric signal matrices by  $\sigma = \frac{1}{2n}\norm{\tilde{\mat{M}}_{1} - \tilde{\mat{M}}_{2}}_{F}$.

Then, we can estimate the spectrum of the ideal signal matrix $\lambda_{1}, \ldots, \lambda_{d}$  from the spectrum of the estimated signal matrix using universal singular value thresholding~\cite{Chatterjee:2012} as $\lambda_{i} = \max(\tilde{\lambda}_{i} - 2\sigma\sqrt{n}, 0)$.
The rank of $\mat{M}$ is determined by the smallest $(i+1)$ for which $\tilde{\lambda}_{i} \leq 2\sigma\sqrt{n}$.
Although we only require the spectra of the ideal signal matrices for computing the meta-embedding coefficients, if we wish we could reconstruct $\mat{M} = \mat{U}\mat{D}\mat{V}\T$, where $\mat{D} = \mathrm{diag}(\lambda_{1}, \ldots, \lambda_{d})$, $\mat{U}$ and $\mat{V}$ are random orthonormal matrices, obtained via SVD on a random matrix of suitable shape.

\section{Experiments}
\label{sec:exp}

\begin{table}[t!]
\small
\centering
\begin{tabular}{clccc}\toprule
&  Method & SimLex & SimVerb & PoS \\ \midrule
\multirow{3}{*}{\rotatebox[origin=c]{90}{source}}  & GloVe & 22.66 & 9.95 & 77.06 \\
&  SGNS & 29.93 & 12.47 & 87.28 \\
&  LSA & 32.64 & 21.01 & 85.22 \\ \\
\multirow{4}{*}{\rotatebox[origin=c]{90}{meta}} & UW & 37.37 & 21.37 & 87.92 \\
&  AVG & 36.11 & 19.47 & 87.88 \\
& SVD & 36.22 & 19.55 & 87.71 \\
&  SW & 35.97 & 21.59 & 88.11 \\
&  DW & \textbf{39.12} & \textbf{23.99} & \textbf{88.63} \\ \bottomrule
\end{tabular}
\caption{Comparing different weighting methods.}
\label{tbl:weights}
\vspace{-5mm}
\end{table}

\begin{table*}[t]
\small
\centering
\csvreader[tabular=p{10mm} p{7mm} p{7mm} p{7mm} p{7mm} p{7mm} p{7mm} p{7mm} p{7mm} p{7mm} p{7mm} p{7mm} p{7mm}, 
	table head=\toprule  \textbf{Method} &  \textbf{WS} & \textbf{RG} & \textbf{RW} & \textbf{MEN} & \textbf{GL}  & \textbf{MSR} & \textbf{MR} & \textbf{CR} & \textbf{MPQA} & \textbf{Ent} & \textbf{SICK} & \textbf{STS}\\ \midrule,
	late after last line=\\ \bottomrule,
	]%
	{bigtable.csv}{Embedding=\Embedding, ws=\ws, rg=\rg, rw=\rw, men=\men, Google=\Google, MSR=\MSR, MR=\MR, CR=\CR, MPQA=\MPQA, SICKEntailment=\SICKEntailment, SICKRelatedness=\SICKRelatedness, STSBenchmark=\STSBenchmark}%
	{\Embedding & \ws & \rg & \rw & \men & \Google & \MSR & \MR & \CR & \MPQA & \SICKEntailment & \SICKRelatedness & \STSBenchmark}%
	\caption{Performance of source embeddings (top) baselines/prior work (middle) and concatenated meta-embeddings (bottom) for tasks described in \autoref{sec:SOTA}.
	Best performance for each dataset is shown in bold, whereas * denotes when this is statistically significantly better than the second best method for the same dataset.}
	\label{tbl:res}
	\vspace{-5mm}
\end{table*}

\subsection{Effect of Weighted Concatenation}
\label{sec:weight}

To evaluate the different weighting methods, we create the meta-embeddings of the following three source embeddings:
\textbf{GloVe} embeddings trained on the Wikipedia Text8 corpus (17M tokens)~\cite{text8},
\textbf{SGNS} embeddings trained on the WMT21 English News Crawl (206M tokens),\footnote{\url{http://statmt.org/wmt21/translation-task.html}} and
\textbf{LSA} embeddings trained on an Amazon product review corpus (492M tokens)~\cite{ni-etal-2019-justifying}.

For GloVe, the elements of the signal matrix $M_{ij}$ are computed as $\log(X_{ij})$ where $X_{ij}$ is the frequency of co-occurrence between word $w_{i}$ and $w_{j}$.
The elements of the signal matrix for SGNS are set to $M_{ij} = \mathrm{PMI}(w_{i}, w_{j}) - \log\beta$, where $\mathrm{PMI}(w_{i}, w_{j})$ is the PMI between $w_{i}$ and $w_{j}$.
For LSA, the elements of its signal matrix are set to $M_{ij} = \max(\mathrm{PMI}(w_{i}, w_{j}), 0)$.
Source embeddings are created by applying SVD on the signal matrices. 

The optimal dimensionality of a source embedding is determined by selecting the dimensionality that minimises the PIP loss between the computed and ideal source embedding matrices for varying dimensionalities upto their estimated rank.
The estimated optimal dimensionality for GloVe ($k_{1} = 1707$) is significantly larger than that for SGNS ($k_{2} = 303$) and LSA ($k_{3} = 220$).
This is because the estimated noise for GloVe from the Text8 corpus is significantly smaller ($\sigma = 0.0891$) compared to that for SGNS ($\sigma = 0.3292$) and LSA ($\sigma = 0.3604$), enabling us to fit more dimensions for GloVe for the same cost in bias.
We used the MTurk-771~\cite{Halawi:2012} as a development dataset to estimate the co-occurrence window (set to 5 tokens) and $\beta $ (set to 3) in our experiments.
Moreover, $\alpha$ is set to $0.5$ for all source embeddings, which reported the best performance on MTurk-771.
(Further experimental results on PIP losses and source embedding dimensionalities are given in the Supplementary)

We compare the proposed source-weighted (\textbf{SW}) and dimension-weighted (\textbf{DW}) concatenation methods against several baselines:
unweighted concatenation (\textbf{UW}) where we concatenate the source embeddings without applying any weighting, 
averaging (\textbf{AVG}) the source embeddings as proposed by~\newcite{Coates:NAACL:2018} after padding zeros to sources with lesser dimensionalities, and applying \textbf{SVD} on \textbf{UW} to reduce the dimensionality to 200, which reported the best performance on the development data.

\autoref{tbl:weights} shows the Spearman correlation coefficient on two true word similarity datasets -- SimLex and SimVerb, which are known to be reliable true similarity datasets~\cite{Faruqui:2016,batchkarov-etal-2016-critique}.
To evaluate meta-embeddings for Part-of-Speech (PoS) tagging as a downstream task, we initialise an LSTM with pretrained source/meta embeddings and train a PoS tagger using the CoNLL-2003 train dataset.
PoS tagging accuracy on the CoNLL-2003 test dataset is reported for different embeddings in \autoref{tbl:weights}. 
Among the three sources, we see significant variations in performance over the tasks, where LSA performs best on SimLex and SimVerb, while SGNS on PoS.
This is due to the size and quality of corpora used to train the source embeddings, and simulate a challenging meta-embedding learning problem.
UW remains a strong baseline, outperforming AVG and SVD, and even SW in SimLex.
However, SW outperforms UW in SimVerb and PoS tasks.
Overall, DW reports best performance on all tasks demonstrating its superior ability to learn accurate meta-embeddings even in the presence of weak source embeddings.

\subsection{Comparisons against Prior Work}
\label{sec:SOTA}

We compare against previously proposed meta-embedding learning methods such as \textbf{LLE}, \textbf{1\texttt{TO}N}, \textbf{DAEME}, \textbf{CAEME}, \textbf{AAEME} described in \autoref{sec:related}.
Unfortunately, prior work  have used different sources and corpora which makes direct comparison of published results impossible.
To conduct a fair evaluation, we train GloVe ($k=736, \sigma = 0.1472$), SGNS ($k=121, \sigma = 0.3566$) and LSA ($k=119, \sigma=0.3521$) on the Text8 corpus as the source embeddings.
For prior methods, we use the publicly available implementations by the original authors.
The average run times of SW and DW are ca. 30 min (wall clock time) measured on a EC2 p3.2xlarge instance.
Hyperparameters for those methods were tuned on MTurk-771 as reported in the Supplementary.

In \autoref{tbl:res}, we use evaluation tasks and benchmark datasets used in prior work: 
(1) \emph{word similarity prediction} (Word Similarity-353 (WS), Rubenstein-Goodenough (RG), rare words (RW), Multimodal Distributional Semantics (MEN),
(2) \emph{analogy detection} (Google analogies (GL), Microsoft Research syntactic analogy dataset (MSR),
(3) \emph{sentiment classification} (movie reviews (MR), customer reviews (CR), opinion polarity (MPQA)),
(4) \emph{semantic textual similarity} benchmark (STS),
 (5) \emph{textual entailment} (Ent) and (6) \emph{relatedness} (SICK).
 For computational convenience, we limit the vocabulary to the most frequent 20k words. 
Note that this is significantly smaller compared to vocabulary sizes and training corpora used in publicly available word embeddings
(e.g. GloVe embeddings are trained on 42B tokens from Common Crawl corpus covering 2.8M word vocabulary).
Therefore, the absolute performance scores we report here cannot be compared against SoTA on these benchmark.
The goal in our evaluation is to compare the relative performances among the different meta-embedding methods.

From \autoref{tbl:res} we see that the proposed SW and DW methods report the best performance in all benchmarks.
In particular, the performance of DW in RW, CR and SW in GL, MSR are statistically significant over the second best methods for those datasets.
Among the three sources, we see that SGNS and LSA perform similarly in most tasks. 
Given that both SGNS and LSA use variants of PMI\footnote{when $\beta =1$ SPMI becomes PMI} this behaviour is to be expected.
However, we note that the absolute performance of source embeddings usually improves when trained on larger corpora and thus what is important in meta-embedding evaluation is not the \emph{absolute} performance of the sources, but how much \emph{relative} improvements one can obtain by applying a meta-embedding algorithm on top of the pre-trained source embeddings.
In this regard, we see that in particular \textbf{DW} significantly improves over the best input source embedding in all tasks, except in \textbf{RG} where the improvement over \textbf{LSA} is statistically insignificant.

We see that UW is a strong baseline, outperforming even more complex prior proposals.
By learning unsupervised concatenation weights, we can further improve UW as seen by the performance of SW and DW.
In particular, DW outperforms SW in all datasets, except in RG, GL and MSR. RG is a smaller (65 word pairs) dataset and the difference in performance between SW and DW there is not significant. 
Considering that GL and MSR are datasets for word analogy prediction, we see that SW is particularly better for analogy tasks.

\begin{table}[t]
\small
\centering
\begin{tabular}{clp{7mm} p{7mm} p{7mm} p{7mm} p{7mm}} \toprule
& Method & MEN & STS13 & STS14 & STS15 & STS16 \\ \midrule
\parbox[t]{1mm}{\multirow{6}{*}{\rotatebox[origin=c]{90}{GloVe+SGNS}}} & UW & 72.45	 &45.13	& 49.91	& 54.18	& 41.93 \\
& SW & 72.64		& 43.86	& 48.83	& 53.17	& 41.04  \\
& DW & $\mathbf{74.27}$		& $\mathbf{46.37}$	& $\mathbf{51.93}$	& $\mathbf{55.84}$	& $\mathbf{44.36}$ \\
& CAEME & 72.63 & 		43.11 &	48.46 &	53.48 &	41.56 \\
& DAEME & 64.51 &		42.15 &	48.69 &	53.33 &	41.93 \\
& AAEME & 71.70	& 		44.24 &	48.39 &	53.32 &	41.67 \\ \midrule
\parbox[t]{1mm}{\multirow{6}{*}{\rotatebox[origin=c]{90}{GloVe+LSA}}} & UW &73.58 &	44.95	&49.95	&54.52	&42.52\\
& SW & 73.30	&43.85	&48.91	&53.33	&41.23\\
& DW & $\mathbf{74.21}$	&$\mathbf{46.39}$	&$\mathbf{52.3}$ &	$\mathbf{56.38}$	&$\mathbf{44.63}$\\
& CAEME & 73.35 &	43.22	&48.85	&54.05	&42.03\\
& DAEME & 64.02 &	42.63	&49.21	&53.35	&42.09\\
& AAEME & 72.40 &	44.63	&49.01	&53.92	&42.62 \\\midrule
\parbox[t]{1mm}{\multirow{6}{*}{\rotatebox[origin=c]{90}{SGNS+LSA}}} & UW & 72.21	& 45.23	& 49.94	& 53.13	& 42.47\\
& SW & 71.81 &	45.28	& 49.96 &	52.81 &	42.14 \\
& DW & $\mathbf{73.10}$	& $\mathbf{46.48}$	& $\mathbf{51.47}$	& $\mathbf{53.85}$ &	$\mathbf{43.73}$\\
& CAEME & 72.21	 & 43.77	& 47.90	& 53.03 &	41.42 \\
& DAEME & 71.13 &	43.50	& 47.74	& 52.57	& 41.98\\
& AAEME & 72.13 &	43.86	& 48.12	& 53.16	& 41.65 \\ \bottomrule
\end{tabular}
\caption{Pairwise meta-embedding using two sources.}
\vspace{-5mm}
\label{tbl:pairs}
\end{table}

SVD and AVG perform even worse than some of the source embeddings due to the discrepancy of the dimensionalities in the sources.
Prior work have used identical or similar (in a narrow range [100,300]) dimensional sources, hence did not observe the degradation of performance when the sources are of different dimensionalities.
For example, AEME pads zeros as necessary to the source embeddings to convert them to the same dimensionality before averaging, introducing extra noise. 
1\texttt{TO}N learns projection matrices between the meta-embedding space and each of the source embedding spaces.
Therefore, when the sources are of different dimensionalities the projection matrices have different numbers of parameters, requiring careful balancing.
We recommend that future work evaluate the robustness in performance considering the optimal dimensionalities of the sources.
To ablate sources, we consider all pairwise meta-embeddings in \autoref{tbl:pairs}, which shows the performance on word similarity (MEN) and multiple STS benchmarks (STS 13, 14, 15 and 16).
We see that DW consistently outperforms all other methods, showing the effectiveness of dimension-weighing when meta-embedding sources of different dimensionalities and performances.

\section{Conclusion}

We showed that concatenated meta-embeddings minimise PIP loss and proposed two weighted-concatenation methods. 
Experiments showed that the dimension-weighted meta-embeddings outperform prior work on multiple tasks.
We plan to extend this result to contextualised embeddings.

\section{Ethical Considerations}

This work consisted of both theoretical contributions and empirical analysis related to word meta-embedding learning.
As already described in the paper, meta-embeddings learnt using the proposed method can be used as input representations in a wide-range of NLP applications such as named entity recognition, text classification, document clustering, sentiment analysis and information extraction.
Therefore, the impact and implications of a meta-embedding learning method on NLP community is significant.
Although we did not annotate any datasets by ourselves we have used multiple benchmark datasets that have been collected, annotated and repeatedly evaluated in prior work. 
We direct the interested readers to the original papers (cited in the Appendix) for the processes undertaken for data collection and annotation.
To the best of our knowledge, no ethical issues have been reported concerning these benchmark datasets.

Word embeddings have known to encode and even at times amplify discriminative biases such as gender, racial or religious biases~\cite{shah-etal-2020-predictive,kaneko-bollegala-2019-gender}.
Therefore, NLP applications that use word embeddings as input features can also be biased to different degrees.
Given that meta-embeddings combine multiple pretrained word embedding sources, it remains an important and interesting research question whether such unfair biases are also encoded or amplified in the meta-embeddings of the biased source embeddings.

All experiments and evaluations conducted in this paper were focused on English language. 
Therefore, verifying these claims on other languages before the meta-embedding methods proposed in the paper on downstream tasks that involve languages other than English is an important next step.

\appendix
\section*{Appendix}

\section{Proof of Theorem 1 in the Main Paper}
\label{sec:proof}

In this section, we prove Theorem~1 we stated in the main paper. 
For the ease of reference, we restate Theorem~\ref{th:main} below.
\begin{theorem}
\label{th:main}
For two source embeddings matrices $\mat{E}_{1}$ and $\mat{E}_{2}$, the PIP loss between their dimension-weighted meta-embedding $\hat{\mat{E}} = [\mat{E}_{1}\mat{C}_{1}; \mat{E}_{2}\mat{C}_{2}]$ and an ideal meta-embedding $\mat{E}$ is given by \eqref{eq:PIP-ME}.
\begin{align}
\label{eq:PIP-ME}
& \norm{\PIP{\mat{E}} - \PIP{\hat{\mat{E}}}}_{F} \nonumber \\
& \leq  \sqrt{\sum_{i=k_{1} + k_{2}}^{d}\lambda_{i}^{4\alpha}} \nonumber \\
& + \sqrt{2}\sum_{i=1}^{k_{1}}\left(\lambda_{i}^{2\alpha} - \lambda_{i+1}^{2\alpha}\right)\norm{\mat{U}^{(1)}_{\cdot,1:i}\T \mat{U}_{\cdot,i:n}}_{F} \nonumber \\
&+ \sqrt{\sum_{i=1}^{k_{1}}\left(\lambda_{i}^{2\alpha} - c^{2}_{1,i}\mu_{i}^{2\alpha} \right)^{2}} \nonumber \\ 
& + \sqrt{\sum_{i=k_{1}+1}^{k_{1}+k_{2}}\left(\lambda_{i}^{2\alpha} - c^{2}_{2,i-k_{1}} \nu_{i-k_{1}}^{2\alpha}\right)^{2}} \nonumber \\
& + \sqrt{2}\sum_{i=k_{1} + 1}^{k_{1} + k_{2}} \left(\lambda_{i}^{2\alpha} - \lambda_{i+1}^{2\alpha} \right) \norm{\mat{U}^{(2)}_{\cdot,1:i}\T \mat{U}_{\cdot,i:n}}_{F}
\end{align}
\end{theorem}

We first prove Lemma~\ref{lem:decomp}, which will be used in the proof of Theorem~\ref{th:main}. An alternative (but a longer) proof of this Lemma was given by \newcite{Yin:2018}, which is shown in \autoref{sec:alt-proof}.
\begin{lemma}
 \label{lem:decomp}
 Let $\mat{X}$, $\mat{Y}$ be two orthonormal matrices of $\R^{n \times n}$. Let $\mat{X} = [\mat{X}_{0}, \mat{X}_{1}]$ and 
 $\mat{Y} = [\mat{Y}_{0}, \mat{Y}_{1}]$ be the first $k$ columns of $\mat{X}$ and $\mat{Y}$ respectively, namely $\mat{X}_{0}, \mat{Y}_{0} \in \R^{n \times k}$ and $k \leq n$. Then, the following holds for the Frobenius norm.
 \begin{align}
 \label{eq:decomp}
 \norm{\mat{X}_{0}\mat{X}_{0}\T - \mat{Y}_{0}\mat{Y}_{0}\T}_{F} = \sqrt{2} \norm{\mat{X}_{0}\T\mat{Y}_{1}}_{F}
 \end{align}
\end{lemma}
\begin{proof}
For a matrix $\mat{A}$, its Frobenius norm, $\norm{\mat{A}}_{F}$, is given by $\sqrt{\tr(\mat{A}\mat{A}\T)}$, where $\tr$ denotes the matrix trace. 
Therefore, we can write the L.H.S. of \eqref{eq:decomp} as follows.

{\small
\begin{align}
 &\norm{\mat{X}_{0}\mat{X}_{0}\T - \mat{Y}_{0}\mat{Y}_{0}\T}_{F} \nonumber \\
 =& \bigl \{ \tr \left(\mat{X}_{0}\mat{X}_{0}\T - \mat{Y}_{0}\mat{Y}_{0}\T\right) {\left(\mat{X}_{0}\mat{X}_{0}\T - \mat{Y}_{0}\mat{Y}_{0}\T\right)}\T \bigr\}^{1/2} \\
 =& \{ \tr ( \mat{X}_{0}\mat{X}_{0}\T\mat{X}_{0}\mat{X}_{0}\T - \mat{X}_{0}\mat{X}_{0}\T\mat{Y}_{0}\mat{Y}_{0}\T  \nonumber \\
 & - \mat{Y}_{0}\mat{Y}_{0}\T\mat{X}_{0}\mat{X}_{0}\T + \mat{Y}_{0}\mat{Y}_{0}\T\mat{Y}_{0}\mat{Y}_{0}\T) \}^{1/2} \\
 =& \{ \tr \left( \mat{X}_{0}\mat{X}_{0}\T\mat{X}_{0}\mat{X}_{0}\T\right) - \tr\left(\mat{X}_{0}\mat{X}_{0}\T\mat{Y}_{0}\mat{Y}_{0}\T\right) \nonumber \\
 & - \tr\left(\mat{Y}_{0}\mat{Y}_{0}\T\mat{X}_{0}\mat{X}_{0}\T\right) + \tr\left(\mat{Y}_{0}\mat{Y}_{0}\T\mat{Y}_{0}\mat{Y}_{0}\T \right) \}^{1/2} \label{eq:sqd}
 \end{align}
 }%
 
 Because a subset of columns in orthonormal matrices $\mat{X}$ and $\mat{Y}$ will be orthogonal, we have $\mat{X}_{0}\T\mat{X}_{0} = \mat{I}_{n}$ and $\mat{Y}_{0}\T\mat{Y}_{0} = \mat{I}_{n}$, where $\mat{I}_{n} \in \R^{n \times n}$ is the identity matrix.
 Moreover, from the cyclic property of matrix transpose we have, $\tr\left(\mat{Y}_{0}\mat{Y}_{0}\T\mat{X}_{0}\mat{X}_{0}\T\right) = \tr\left(\mat{X}_{0}\mat{X}_{0}\T\mat{Y}_{0}\mat{Y}_{0}\T\right)$. Substituting these results in \eqref{eq:sqd} we obtain the following.
 
 {\small
 \begin{align}
 & \norm{\mat{X}_{0}\mat{X}_{0}\T - \mat{Y}_{0}\mat{Y}_{0}\T}_{F} \nonumber \\
 &= \{ \tr(\mat{X}_{0}\mat{X}_{0}\T) - 2\tr(\mat{X}_{0}\mat{X}_{0}\T\mat{Y}_{0}\mat{Y}_{0}\T) + \tr(\mat{Y}_{0}\mat{Y}_{0}\T) \}^{1/2} \label{eq:sq2}
 \end{align} 
 }%
 
 Note that from the orthonormality of $\mat{Y}$ we have the following.
 \begin{align}
 \mat{Y}\T\mat{Y} &= \mat{I}_{n} \\
 \tr(\mat{Y}\T\mat{Y}) &= \tr(\mat{I}_{n}) \\
 \tr(\mat{Y}\mat{Y}\T) &= n \\
 \tr(\mat{Y}_{0}\mat{Y}_{0}\T + \mat{Y}_{1}\mat{Y}_{1}\T) &= n \\
\tr(\mat{Y}_{0}\mat{Y}_{0}\T) + \tr(\mat{Y}_{1}\mat{Y}_{1}\T) &= n \label{eq:n} 
\end{align}

Substituting \eqref{eq:n} in \eqref{eq:sq2} we obtain,
{\small
\begin{align}
 &\norm{\mat{X}_{0}\mat{X}_{0}\T - \mat{Y}_{0}\mat{Y}_{0}\T}_{F} \nonumber \\
 & = \sqrt{\tr(\mat{X}_{0}\mat{X}_{0}\T) -2\tr(\mat{X}_{0}\mat{X}_{0}\T(\mat{I}_{n} - \mat{Y}_{1}\mat{Y}_{1}\T)) + \tr(\mat{Y}_{0}\mat{Y}_{0}\T)} . \label{eq:q1}
 \end{align}
 }%
 
Substituting the results $\tr(\mat{X}_{0}\mat{X}_{0}\T) = \tr(\mat{X}_{0}\T\mat{X}_{0}) = \tr(\mat{I}_{k}) = k$ and
$\tr(\mat{Y}_{0}\mat{Y}_{0}\T) = \tr(\mat{Y}_{0}\T\mat{Y}_{0}) = \tr(\mat{I}_{k}) = k$ in \eqref{eq:q1} and cancelling out terms we obtain,
\begin{align}
&\norm{\mat{X}_{0}\mat{X}_{0}\T - \mat{Y}_{0}\mat{Y}_{0}\T}_{F} \nonumber \\
&= \sqrt{2 \tr(\mat{X}_{0}\mat{X}_{0}\T\mat{Y}_{1}\mat{Y}_{1}\T)} \\
&= \sqrt{2} \sqrt{\tr( (\mat{X}_{0}\T\mat{Y}_{1}) (\mat{X}_{0}\T\mat{Y}_{1})\T)} \\
&= \sqrt{2} \norm{\mat{X}_{0}\T\mat{Y}_{1}}_{F} \label{eq:q2}. 
\end{align}
\end{proof}

Armed with Lemma~\ref{lem:decomp}, we are now in a position to prove \autoref{th:main}
\begin{proof} (\autoref{th:main})
To reduce the complexity of the notation, we define $\mat{Y}_{0} = \mat{U}^{(1)}_{\cdot, 1:{k}_{1}}$,  $\mat{D}_{1:{k}_{1}, 1:{k}_{1}} = \mathbf{\Lambda}_{0}$,
$\mat{Z}_{0} = \mat{U}^{(2)}_{\cdot, 1:{k}_{2}}$, $\mat{D}_{1:{k}_{2}, 1:{k}_{2}} = \mathbf{\Gamma}_{0}$, $\mat{X}_{0} = \mat{U}_{1:k,1:k}$ and $\mat{D}_{0} = \mat{D}_{1:k,1:k}$. Then the source and ideal embeddings can be written as follows.
\begin{align}
& \mat{E}_{1} =   \mat{U}^{(1)}_{\cdot, 1:{k}_{1}} \mat{D}_{1:{k}_{1}, 1:{k}_{1}} = \mat{Y}_{0}  \mathbf{\Lambda}_{0} \\
& \mat{E}_{2} =   \mat{U}^{(2)}_{\cdot, 1:{k}_{2}} \mat{D}_{1:{k}_{2}, 1:{k}_{2}} = \mat{Z}_{0}  \mathbf{\Gamma}_{0} \\
& \mat{E} = \mat{U}_{1:k,1:k} \mat{D}_{1:k,1:k} = \mat{X}_{0} \mat{D}_{0} 
\end{align}
Moreover, let $\mat{X}_{0} = [\mat{X}_{0,1}, \mat{X}_{0,2}, \mat{X}_{0,3}]$, where $\mat{X}_{0,1} \in \R^{n \times k_{1}}$, $\mat{X}_{0,2} \in \R^{n \times k_{2}}$ and $\mat{X}_{0,3} \in \R^{n \times (d - k_{1} - k_{2})}$.
Likewise, we can decompose $\mat{D}_{0} = \mathrm{diag}(\lambda_{1}, \ldots, \lambda_{d})$ into three diagonal matrices $\mat{D}_{0,1} = \mathrm{diag}(\lambda_{0}, \ldots, \lambda_{{k}_{1}})$, $\mat{D}_{0,2} = \mathrm{diag}(\lambda_{{k}_{1} + 1}, \ldots, \lambda_{{k}_{1}+{k}_{2}})$ and $\mat{D}_{0,3} = \mathrm{diag}(\lambda_{{k}_{1} + {k}_{2} + 1}, \ldots, \lambda_{d})$. 
Here, the singular values are arranged in the descending order, $\lambda_{1} \geq \lambda_{2}, \cdots, \lambda_{d}$.
Note that the diagonal matrix $\mat{D}_{0}$ can be expressed using a telescopic sum using identity matrices $\mat{I}_{i} \in \R^{i \times i}$ as in \eqref{eq:tele}.
\begin{align}
 \label{eq:tele}
 D_{0}^{\alpha} = \sum_{i=1}^{d}(\lambda^{\alpha}_{i} - \lambda^{\alpha}_{i+1}) \mat{I}_{i}
\end{align}
Here, we have defined $\lambda_{d+1} = 0$.

We can now evaluate the PIP loss as follows:
{\small
\begin{align}
&\norm{\PIP{\mat{E}} - \PIP{\hat{\mat{E}}}}  \nonumber \\
&= \norm{\mat{E}\mat{E}\T - \hat{\mat{E}}\hat{\mat{E}}\T} \\
&=\norm{\mat{X}_{0}\mat{D}^{\alpha}_{0} {\left(\mat{X}_{0}\mat{D}^{\alpha}_{0}\right)}\T - \begin{bmatrix} \mat{E}_{1}\mat{C}_{1} & \mat{E}_{2}\mat{C}_{2} \end{bmatrix} \begin{bmatrix}\mat{C}_{1}\mat{E}_{1}\T \\ \mat{C}_{2}\mat{E}_{2}\T \end{bmatrix}} \\
&=|| \begin{bmatrix} \mat{X}_{0,1}\mat{D}^{\alpha}_{0,1} & \mat{X}_{0,2}\mat{D}^{\alpha}_{0,2} & \mat{X}_{0,3}\mat{D}^{\alpha}_{0,3} \end{bmatrix} \begin{bmatrix} \mat{D}^{\alpha}_{0,1} \mat{X}_{0,1}\T \\ \mat{D}^{\alpha}_{0,2}\mat{X}_{0,2}\T \\ \mat{D}^{\alpha}_{0,3}\mat{X}_{0,3}\T \end{bmatrix} \nonumber \\
& - \begin{bmatrix} \mat{Y}_{0}\mathbf{\Lambda}^{\alpha}_{0}\mat{C}_{1} & \mat{Z}_{0}\mathbf{\Gamma}^{\alpha}_{0}\mat{C}_{2} \end{bmatrix} \begin{bmatrix} \mat{C}_{1}\mathbf{\Lambda}^{\alpha}_{0}\mat{Y}_{0}\T \\ \mat{C}_{2}\mathbf{\Gamma}^{\alpha}_{0}\mat{Z}_{0}\T \end{bmatrix} || \\
&=|| \mat{X}_{0,1}\mat{D}^{2\alpha}_{0,1} \mat{X}_{0,1}\T + \mat{X}_{0,2}\mat{D}^{\alpha}_{0,2}\mat{X}_{0,2}\T + \mat{X}_{0,3}\mat{D}^{\alpha}_{0,3}\mat{X}_{0,3}\T \nonumber \\
& \ \ \ - \mat{Y}_{0}\mathbf{\Lambda}^{\alpha}\mat{C}^{2}_{1}\mathbf{\Lambda}^{\alpha}_{0}\mat{Y}_{0}\T - \mat{Z}_{0}\mathbf{\Gamma}^{\alpha}_{0}\mat{C}^{2}_{2}\mathbf{\Gamma}^{\alpha}_{0}\mat{Z}_{0}\T || \label{eq:t1} 
\end{align}
}%

Considering that the product of diagonal matrices are also diagonal matrices, we can write 
\begin{align}
&\mathbf{\Lambda}^{\alpha}\mat{C}^{2}_{1}\mathbf{\Lambda}^{\alpha}_{0} = \mathrm{diag}(\mu^{2\alpha}_{1}c^{2}_{1,1}, \ldots, \mu^{2\alpha}_{{k}_{1}}c^{2}_{1,{k}_{1}}) = \tilde{\mat{D}}_{1}, \label{eq:D1}\\
&\mathbf{\Gamma}^{\alpha}\mat{C}^{2}_{2}\mathbf{\Gamma}^{\alpha}_{0} = \mathrm{diag}(\nu^{2\alpha}_{1}c^{2}_{2,1}, \ldots, \nu^{2\alpha}_{{k}_{2}}c^{2}_{2,{k}_{2}}) = \tilde{\mat{D}}_{2} \label{eq:D2}.
\end{align}
Substituting \eqref{eq:D1} and \eqref{eq:D2} in \eqref{eq:t1} we obtain the following.
{\small
\begin{align}
&\norm{\PIP{\mat{E}} - \PIP{\hat{\mat{E}}}}  \nonumber \\
&=|| \mat{X}_{0,1}\mat{D}^{2\alpha}_{0,1} \mat{X}_{0,1}\T + \mat{X}_{0,2}\mat{D}^{\alpha}_{0,2}\mat{X}_{0,2}\T \nonumber \\
&+ \mat{X}_{0,3}\mat{D}^{\alpha}_{0,3}\mat{X}_{0,3}\T - \mat{Y}_{0}\tilde{\mat{D}}_{1}\mat{Y}_{0}\T - \mat{Z}_{0}\tilde{\mat{D}}_{2}\mat{Z}_{0}\T || \\
&\leq \underbrace{\norm{\mat{X}_{0,3}\mat{D}^{2\alpha}_{0,3}\mat{X}_{0,3}\T}}_{\text{Term 1}} \nonumber \\
&+ \underbrace{\norm{\mat{X}_{0,1}\mat{D}^{2\alpha}_{0,1}\mat{X}_{0,1}\T - \mat{Y}_{0}\tilde{\mat{D}}_{1}\mat{Y}_{0}\T}}_{\text{Term 2}} \nonumber \\
& + \underbrace{\norm{\mat{X}_{0,2}\mat{D}^{2\alpha}_{0,2}\mat{X}_{0,2}\T - \mat{Z}_{0}\tilde{\mat{D}}_{2}\mat{Z}_{0}\T}}_{\text{Term 3}}
\end{align}
}%

Next, lets compute each of the three terms separately.

For Term 1,
{\small
\begin{align}
 &\norm{\mat{X}_{0,3}\mat{D}^{2\alpha}\mat{X}_{0,3}\T} \nonumber \\
 &= \sqrt{\tr\left( \mat{X}_{0,3}\mat{D}^{2\alpha}_{0,3}\underbrace{\mat{X}_{0,3}\T\mat{X}_{0,3}}_{=\mat{I}_{d-{k}_{1}-{k}_{2}}}\mat{D}^{2\alpha}_{0,3}\mat{X}_{0,3}\T \right)} \\
 &= \sqrt{\tr(\mat{X}_{0,3}\mat{D}^{4\alpha}_{0,3}\mat{X}_{0,3}\T)} \\
 & = \sqrt{\tr(\underbrace{\mat{X}_{0,3}\T \mat{X}_{0,3}}_{=\mat{I}_{d-{k}_{1}-{k}_{2}}}\mat{D}^{4\alpha}_{0,3})} \\
 &= \sqrt{\tr(\mat{D}^{4\alpha}_{0,3})} \\
 &= \sqrt{\sum_{i={k}_{1}+{k}_{2}}^{d} \lambda^{4\alpha}_{i}} \label{eq:T1}
\end{align}
}%

Term 2 can be further decomposed to two terms as follows.
{\small
\begin{align}
 &\norm{\mat{X}_{0,1}\mat{D}^{2\alpha}_{0,1}\mat{X}_{0,1}\T - \mat{Y}_{0}\tilde{\mat{D}}_{1}\mat{Y}_{0}\T} \nonumber \\
 &=  \norm{\mat{X}_{0,1}\mat{D}^{2\alpha}_{0,1}\mat{X}_{0,1}\T - \mat{Y}_{0}\mat{D}^{2\alpha}_{0,1}\mat{Y}_{0}\T + \mat{Y}_{0}\mat{D}^{2\alpha}_{0,1}\mat{Y}_{0}\T -  \mat{Y}_{0}\tilde{\mat{D}}_{1}\mat{Y}_{0}\T}\\
 &\leq \underbrace{\norm{\mat{X}_{0,1}\mat{D}^{2\alpha}_{0,1}\mat{X}_{0,1}\T - \mat{Y}_{0}\mat{D}^{2\alpha}_{0,1}\mat{Y}_{0}\T}}_{\text{Term 2.1}} + \underbrace{\norm{\mat{Y}_{0}\mat{D}^{2\alpha}_{0,1}\mat{Y}_{0}\T -  \mat{Y}_{0}\tilde{\mat{D}}_{1}\mat{Y}_{0}\T}}_{\text{Term 2.2}}
\end{align}
}%

Using the telescoping sum for the diagonal matrices, Term 2.1 can be evaluated as follows.
{\small
\begin{align}
& \norm{\mat{X}_{0,1}\mat{D}^{2\alpha}_{0,1}\mat{X}_{0,1}\T - \mat{Y}_{0}\mat{D}^{2\alpha}_{0,1}\mat{Y}_{0}\T} \nonumber \\
&=  \norm{\sum_{i=1}^{k_{1}} (\lambda^{2\alpha}_{i} - \lambda^{2\alpha}_{i+1}) \left( \mat{X}_{\cdot, 1:i}\mat{X}\T_{\cdot,1:i} - \mat{Y}_{\cdot,1:i}\mat{Y}\T_{\cdot,1:i} \right)} \\
&\leq \sum_{i=1}^{k_{1}} (\lambda^{2\alpha}_{i} - \lambda^{2\alpha}_{i+1}) \norm{\mat{X}_{\cdot, 1:i}\mat{X}\T_{\cdot,1:i} - \mat{Y}_{\cdot,1:i}\mat{Y}\T_{\cdot,1:i}} \\
&= \sqrt{2} \sum_{i=1}^{k_{1}} (\lambda^{2\alpha}_{i} - \lambda^{2\alpha}_{i+1}) \norm{\mat{Y}\T_{\cdot,1:i} \mat{X}_{\cdot,i:n}} \label{eq:t2}
\end{align}
}%
We have applied Lemma~\ref{lem:decomp} to the norm term in \eqref{eq:t2}.

Term 2.2 can be evaluated as follows.
{\small
\begin{align}
& \norm{\mat{Y}_{0}\mat{D}^{2\alpha}_{0,1}\mat{Y}_{0}\T -  \mat{Y}_{0}\tilde{\mat{D}}_{1}\mat{Y}_{0}\T} \nonumber \\
&= \norm{\mat{Y}_{0} \left(\mat{D}^{2\alpha}_{0,1} - \tilde{\mat{D}}_{1} \right) \mat{Y}_{0}\T} \\
&= \sqrt{\sum_{i=1}^{k_{1}} {\left(\lambda^{2\alpha}_{i} - c^{2}_{1,i}\mu^{2\alpha}_{i}\right)}^{2}} \label{eq:t3}
\end{align}
}%

Here, we use the fact that Frobenius norm is invariant under an orthogonal transformation and the difference between two diagonal matrices is also a diagonal matrix.

Considering the resemblance between Terms 2 and 3, Term 3 can be likewise evaluated to arrive at,
{\small
\begin{align}
 \label{eq:T3}
 &\norm{\mat{X}_{0,2}\mat{D}^{2\alpha}_{0,2}\mat{X}_{0,2}\T - \mat{Z}_{0}\tilde{\mat{D}}_{2}\mat{Z}_{0}\T} \nonumber \\
 & \leq   \sqrt{2} \sum_{i={k}_{1} + 1}^{k_{1} + {k}_{2}} (\lambda^{2\alpha}_{i} - \lambda^{2\alpha}_{i+1}) \norm{\mat{Z}\T_{\cdot,1:i} \mat{X}_{\cdot,i:n}} \nonumber \\
 &+ \sqrt{\sum_{i={k}_{1}+1}^{{k}_{1}+{k}_{2}} (\lambda^{2\alpha}_{i} - c^{2}_{2,i-{k}_{1}} \nu^{2\alpha}_{i-{k}_{1}})^{2}} .
\end{align}
}%

Collecting all the terms in \eqref{eq:T1}, \eqref{eq:t2}, \eqref{eq:t3} and \eqref{eq:T3} we can obtain the R.H.S. in \eqref{eq:PIP-ME}.
\end{proof}

\subsection{$N (> 2)$ Source Extensions of \autoref{th:main}}
\label{sec:n-ext}

By decomposing ideal and concatenated matrices source-wise, \autoref{th:main} can be extended for $N$ source embeddings as follows.

{\small
\begin{align}
&\norm{\PIP{\mat{E}} - \PIP{\hat{\mat{E}}}} \nonumber \\
& \leq \sqrt{\sum_{i={\theta}_{N+1}}^{d} \lambda^{4\alpha}_{i}} \nonumber \\ 
&+ \sum_{j=1}^{N+1} \biggl( \sqrt{2}\sum_{i=1+{\theta}_{j}}^{{\theta}_{j+1}} (\lambda^{2\alpha}_{i} - \lambda^{2\alpha}_{i+1}) \norm{\mat{U}^{(j)}_{\cdot,1:i}\T\mat{U}_{\cdot,i:n}} \nonumber \\
&+ \sqrt{\sum_{i=1+{\theta}_{j}}^{\theta_{j+1}} \left(\lambda^{2\alpha} - c^{2}_{j,i}\mu_{i}^{(j)2\alpha}\right)^{2}} \biggr) \label{eq:mult}
\end{align}
}%

Here, we have defined
\begin{align}
\theta_{j} = \begin{cases} 0 & j = 1 \\ \sum_{r=1}^{j-1} k_{r} & \text{otherwise} \end{cases}
\end{align}
and $\mu^{(j)}_{1}, \ldots, \mu^{(j)}_{{k}_{j}}$ is the spectrum of the $j$-th source embedding $s_{j}$.
Proof of \eqref{eq:mult} is similar to \autoref{th:main} and requires adding terms that account for the interaction between the ideal meta-embedding and each source embedding.

\section{Source Embeddings}
\label{supp:sec:sources}

We train source embeddings using three different corpora as follows:
\textbf{GloVe} embeddings trained on the Wikipedia Text8 corpus~\cite{text8} (first 1GB of English Wikipedia covering 17M tokens), 
\textbf{SGNS} embeddings trained on the WMT21 English News Crawl (206M tokens)~\footnote{\url{http://statmt.org/wmt21/translation-task.html}},
\textbf{LSA} embeddings trained on an Amazon product reviews\footnote{\url{https://nijianmo.github.io/amazon/index.html}} corpus~\cite{ni-etal-2019-justifying} (492M tokens). 
For GloVe, the elements of the signal matrix $M_{ij}$ are computed as $\log(X_{ij})$ where $X_{ij}$ is the frequency of co-occurrence between word $w_{i}$ and $w_{j}$.
The elements of the signal matrix for SGNS are set to $M_{ij} = \mathrm{PMI}(w_{i}, w_{j}) - \log\beta$, where $\mathrm{PMI}(w_{i}, w_{j})$ is the PMI between $w_{i}$ and $w_{j}$.
For LSA, the elements of its signal matrix are set to $M_{ij} = \max(\mathrm{PMI}(w_{i}, w_{j}), 0)$.
Source embeddings are created by applying SVD on the signal matrices. 
The optimal dimensionalities for GloVe, SGNS and LSA embeddings are respectively 1714, 185 and 222 and $\sigma$ values are respectively $0.0891$, $0.3661$ and $0.3604$.


\section{Evaluation Datasets}
\label{supp:sec:benchmarks}

In this section, we describe the different NLP tasks and the benchmark datasets used for evaluating the meta-embeddings in the paper.

\paragraph{Word similarity measurement:} 
The correlation between the semantic similarity scores computed using word embeddings, and the corresponding human similarity ratings has been used in prior work to evaluate the accuracy of word embeddings.
We compute the Spearman correlation coefficient between the predicted and human-rated similarities for the word-pairs in several datasets: 
Word Similarity-353 (WS)~\cite{WS353},
Rubenstein and Goodenough's dataset (RG)~\cite{RG},
rare words dataset (RW)~\cite{Luong:CoNLL:2013} and
Multimodal Distributional Semantics (MEN)~\cite{MEN}.
Moreover, we use the MTurk-771~\cite{Halawi:2012} as a development dataset for estimating various hyperparameters in the source embeddings and in previously proposed meta-embedding learning methods.
Fisher's $z$ statistic is used to determine statistical significance under the $p<0.01$ confidence level.

\paragraph{Word analogy detection:} Using the CosAdd method~\cite{Levy:CoNLL:2014}, we solve word-analogy questions in
the Google analogy dataset (GL)~\cite{Mikolov:NIPS:2013} and Microsoft Research syntactic analogy dataset (MSR)~\cite{Mikolov:NAACL:2013}.
Specifically, for three given words $a$, $b$ and $c$, we find a fourth word $d$ 
 such that the cosine similarity between $(\vec{b} - \vec{a} + \vec{c})$ and $\vec{d}$ is maximised. 
 We use Clopper-Pearson confidence intervals~\cite{Clopper-Pearson} computed at $p<0.05$ significance level to determine statistical significance.
 

\paragraph{Sentiment classification:}  
We evaluate meta-embeddings in a sentiment classification task using three datasets:
Movie reviews dataset (MR)~\cite{Pang:ACL:2005},
customer reviews dataset (TR)~\cite{Hu:KDD:2004} and
opinion polarity (MPQA)~\cite{MPQA}.
Given a review, we represent it as the centroid of the unigram embeddings and train a binary logistic regression classifier with a cross-validated $\ell_{2}$ regulariser using the train split of each dataset.
Next, we evaluate its classification accuracy on the test split in each dataset.
 We use Clopper-Pearson confidence intervals~\cite{Clopper-Pearson} computed at the $p<0.05$ significance level to determine statistical significance.

\paragraph{Entailment}
We use the SICK dataset for entailment (SICK-E)~\cite{SICK}, which 
contains 5595 sentence pairs annotated with neutral, 1424 sentence pairs annotated with contradiction and 2821 sentence pairs annotated with entailment. 
We use a meta/source embedding to represent a word in a sentence and create a sentence embedding by averaging the word embeddings for the words in that sentence. 
Next, we train a multi-class classifier (neutral, contradiction and entailment as three classes) to predict the entailment between two sentences.
We use a neural network with one hidden layer, ReLU activations in the hidden layer and a softmax output layer for this purpose.
Dropout (using a rate of 0.5) is applied and the network is optimised via RMSProp. All parameters are tuned using the training and validation parts of the SICK-E and SNLI datasets.
Accuracy is measured as the percentage of the correctly predicted instances to the total number of instances in a dataset and reported as the evaluation measure.
We use Clopper-Pearson confidence intervals~\cite{Clopper-Pearson} computed at the $p<0.05$ significance level to determine statistical significance.

\paragraph{Semantic Textual Similarity:}
Given two snippets of text, semantic textual similarity (STS) captures the notion that some texts are more similar than the others, measuring their degree of semantic equivalence.
A series of STS datasets have been annotated under the SemEval workshops for multiple languages.
We use the datasets from  STS benchmarks: STS13~\cite{agirre-etal-2013-sem}, STS14~\cite{Agirre:SemEval:2014}, STS15~\cite{agirre-etal-2015-semeval} and STS16~\cite{agirre-etal-2016-semeval} in our experiments. STS Benchmark\footnote{\url{http://ixa2.si.ehu.es/stswiki/index.php/STSbenchmark}} compiles several STS datasets into a single dataset, which we also use in our expriments.

 SemEval 2015 Task 2 English subtask contains sentence pairs for the five categories: \emph{Headlines} (750 sentence pairs from newswire headlines), \emph{Image captions} (750 sentence pairs), Answers-student (750 sentence pairs selected from the BEETLE corpus~\cite{dzikovska-etal-2013-semeval}, which is a question-answer dataset collected and annotated during the evaluation of the BEETLE II tutorial dialogue system), Answers-forum (375 sentence pairs selected from Stack Exchange question answer websites), \emph{Belief} (375 sentence pairs from the Belief Annotation dataset (LDC2014E55)).
 
 SemEval 2016 Task 1 English subtask contains sentence pairs for the five categories: \emph{Headlines} (249 sentence pairs from newswire headlines), \emph{Plagiarism} (230 sentence pairs from short-answer plagiarism), \emph{Postediting} (244 sentence pairs from MT postedits), \emph{Ans-Ans} (254 sentence pairs from QA forum answers) and \emph{Quest-Quest} (209 sentence pairs from QA forum questions).
 Each sentence pair is rated between 0 and 5 for the semantic similarity between the two sentences.
 The goal is to evaluate how the cosine distance between two sentences correlate with human-labelled similarity score through Pearson and Spearman correlations.
  We adopt an unsupervised setting, where we do not fit any models to predict the similarity ratings.
 Instead,  we create sentence embeddings by averaging the meta-embeddings for the words in the sentence.
 Following the official evaluation protocol, we report the weighted average (by the number of samples in each subtask) of the Pearson and Spearman correlation.


For the experiments related to word similarity measurement, word analogy detection we use the 
RepsEval evaluation tool\footnote{\url{https://github.com/Bollegala/repseval}} and for the experiments related to entailment and semantic textual similarity we used the SentEval evaluation tool\footnote{\url{https://github.com/facebookresearch/SentEval}} to conduct the evaluations.

\section{Effect of the Robustness Parameter $\alpha$}
\label{sec:alpha}

The parameter $\alpha$ is a hyperparameter of the source word embedding algorithm and not a hyperparameter of any of the two (i.e. source-weighted or dimension-weighted) concatenated meta-embedding methods, which depend only on the eigenvalues of the ideal meta-embedding signal matrix and the signal matrix used by each of the source word embedding algorithms. 
However, the performance of the source embeddings depends on $\alpha$ and this will indirectly influence the performance of their meta-embedding.

Let us first discuss how $\alpha$ influences the performance of a source word embedding. 
Recall that a source embedding matrix $\mat{E}_{j}$ is obtained by decomposing a signal matrix (e.g. log co-occurrences in the case of GloVe) $\mat{M}_{j}$. 
Using the notation defined in the paper, $\mat{E}_{j} = f_{\alpha, {k}_{j}}(\mat{M}_{j}) = \mat{U}^{(j)}_{\cdot, 1:{k}_{j}}\mat{D}^{\alpha}_{1:{k}_{j}, 1:{k}_{j}}$.
Because we are retaining the largest $k_{j}$ singular values of $\mat{M}_{j}$ as the diagonal elements in the diagonal matrix $\mat{D}$, raising $\mat{D}$ to the power $\alpha$ corresponds to the rank $k_{j}$ approximation of $\mat{E}_{j}$ given by \eqref{eq:k-approx}.
\begin{align}
 \label{eq:k-approx}
 \mat{E}_{j} = \sum_{i=1}^{k_{1}}\vec{U}^{(j)}_{\cdot, i} \lambda^{\alpha}_{i}
\end{align}
When $\lambda_{i} \geq 1$, increasing $\alpha (> 0)$ results in emphasising the top left singular vectors in the approximation given by \eqref{eq:k-approx}. On the other hand, when $\lambda \in (0, 1]$, if $\alpha > 1$ increasing $\alpha$ will further diminish the corresponding singular vectors, while $\alpha \in [0, 1]$ will \emph{enlarge} those small directional components.
In this regard, the role played by $\alpha$ is analogous to that of the magnifying factor in a microscope, enlarging or shrinking the influence of the top singular vectors.
The actual values of the singular values will obviously depend on the signal matrix being decomposed. For symmetric co-occurrence matrices, the eigenvalues will be real but could be either positive or negative, and their absolute values could be either greater or smaller than 1. 
Therefore, the optimal value of $\alpha$ must be found for each embedding method separately.   

The earliest analysis of $\alpha$ that we are aware of is by \newcite{Caron:2001}. 
He considered unbounded $\alpha$ values for LSA embeddings (referred to as parameter $p$ in his paper). and reported that $\alpha \in [1,2]$ to produce optimal LSA embeddings and increasing this value beyond 2 or setting negative values to degrade performance in information retrieval systems. 
\newcite{Bullinaria:2012} found that optimal $\alpha$ to be both task and dataset dependent for LSA embeddings.
In their formulation of SGNS as a low-rank decomposition of a signal matrix with shifted PPMI, \newcite{Levy:NIPS:2014} used $\alpha = 0.5$ to obtain symmetric target and context word embeddings. 
They stated that \emph{while it is not theoretically clear why the symmetric approach is better for semantic tasks, it does work much better empirically}.
A theoretical explanation was provided by \newcite{Yin:2018} where they showed that when the dimensionality of the source embedding is large (i.e. when we have already selected the top left singular vectors covering more than half of the energy spectrum, which is mathematically equivalent to $\frac{\sum_{i=1}^{{k}_{j}} \lambda^{2}_{i}}{\sum_{i=1}^{d} \lambda^{2}_{i}} > 0.5$), then when $\alpha \in [0.5, 1]$ PIP loss will be robust to overparametrisation because the variance terms will decrease without significantly affecting the bias term.
We empirically evaluate the effect of $\alpha$ on GloVe, SGNS and LSA embeddings created via SVD on the respective signal matrices in \autoref{fig:alpha}.
We compute the Pearson correlation coefficient between the cosine similarity scores computed using the obtained word embeddings and human similarity ratings for the word-pairs in the MTurk-771~\cite{Halawi:2012} dataset. We see that $\alpha=0.5$ reports the highest correlation in all source embeddings.
This trend was observed for almost all benchmark datasets and tasks used in our experiments and we used $\alpha = 0.5$ for all source embeddings.

\begin{figure}[t]
\centering
\includegraphics[width=8cm]{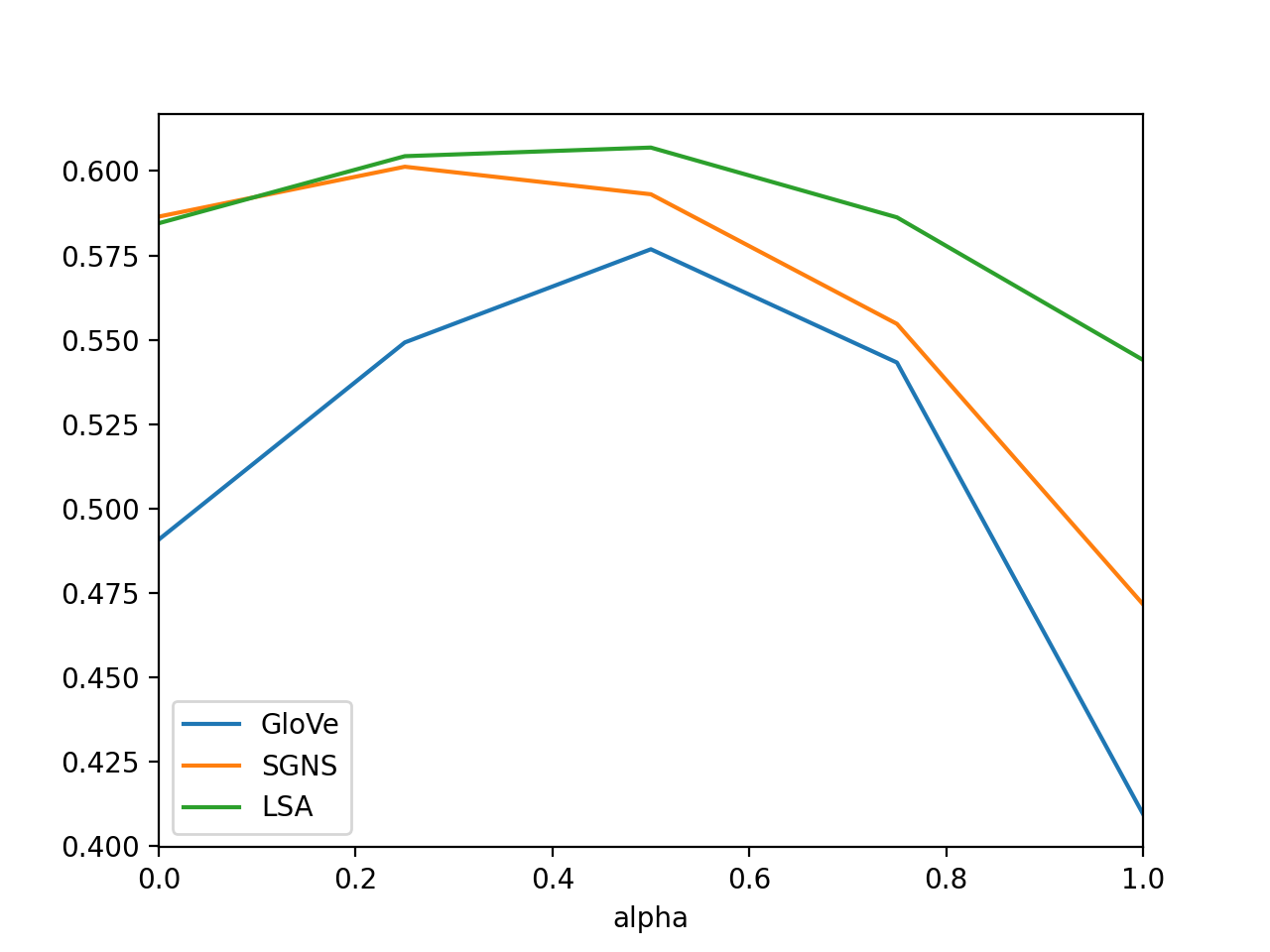}
\caption{Pearson correlation coefficient on MTurk-771 dataset for GloVe, SGNS and LSA embeddings computed with different $\alpha$ values. We see that all sources embeddings obtain their highest correlation when $\alpha=0.5$.}
\label{fig:alpha}
\end{figure}


\section{PIP loss and Optimal Dimensionality}
\label{sec:PIP}

For each of the three source embedding algorithms used in the experiments: GloVe, SGNS and LSA, we create their respective signal matrices and apply SVD with $\alpha = 0.5$ for different dimensionalities $k$. 
The PIP loss between the ideal embedding matrix and the embedding matrix obtained via matrix decomposition is estimated via the sampling procedure described in the paper and are shown for GloVe, SGNS and LSA respectively in Figures~\ref{fig:pip-glove}, \ref{fig:pip-sgns} and \ref{fig:pip-lsa} for different dimensionalities of the embeddings.
We see that PIP loss is convex and there is a unique global minimum corresponding to the optimal dimensionality as expected from a bias-variance trade-off.
Moreover, we see that the optimal dimensionality for GloVe (i.e. 736) is significantly large than that for SGNS (i.e. 121) and LSA (i.e. 119).
This is because the estimated noise for GloVe from the Text8 corpus is significantly smaller (i.e. 0.1472) compared to that for SGNS (i.e 0.3566) and and LSA (i.e. 0.3521) enabling us to fit more dimensions for GloVe for the same cost in bias.
However, when all source embeddings are trained on a much larger 1G text corpus extracted from Wikipedia as opposed to text8, which contains only the first 100MB from Wikiepdia, these differences in dimensionalities become less pronounced (i.e. Glove = 380, SGNS = 180, LSA = 161).

\begin{figure}[t]
\centering
\includegraphics[width=8cm]{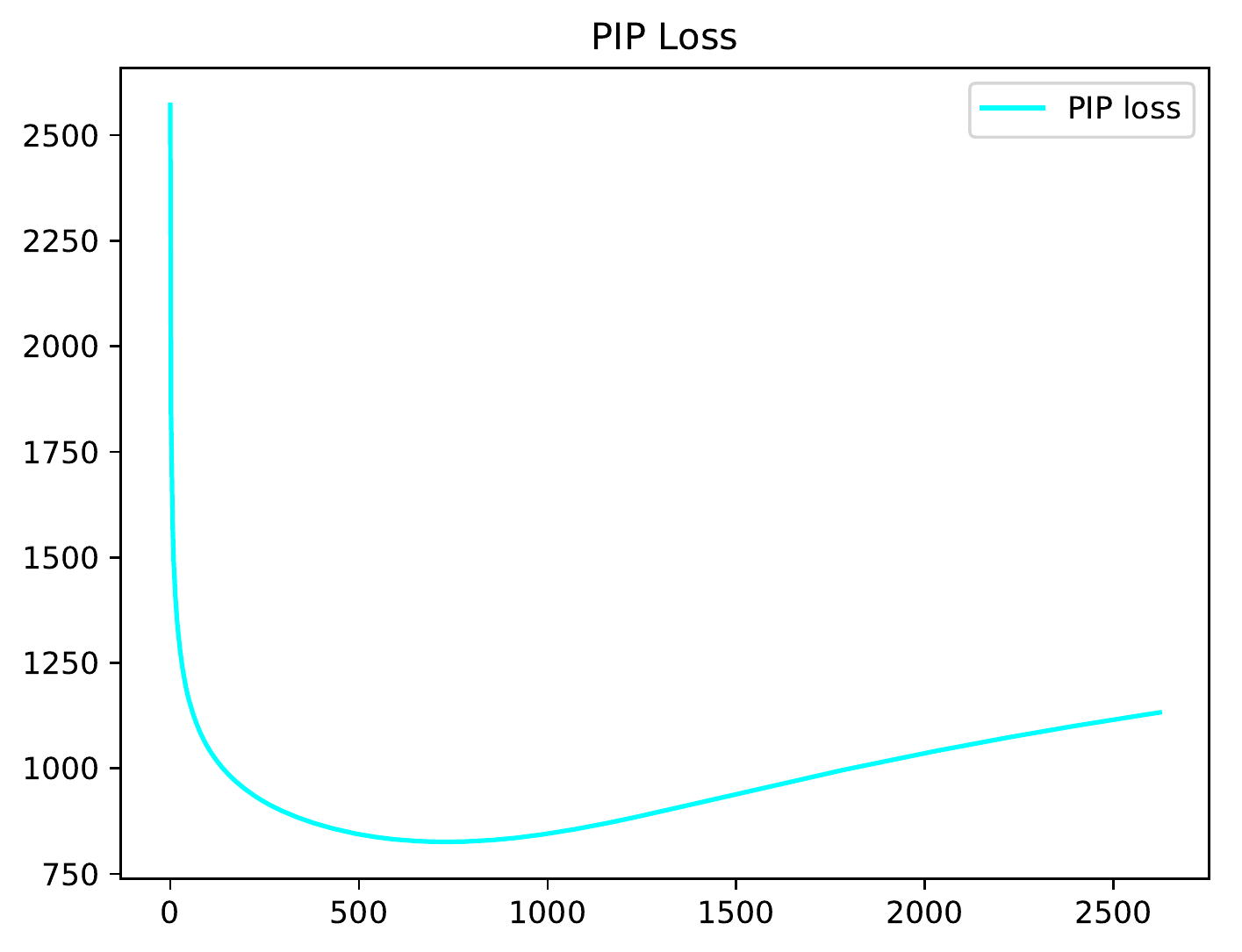}
\caption{PIP loss (in $y$-axis) against the dimensionality (in $x$-axis) for GloVe source embeddings computed with $\alpha=0.5$. The optimal dimensionality is 736.}
\label{fig:pip-glove}
\end{figure}
\begin{figure}[t]
\centering
\includegraphics[width=8cm]{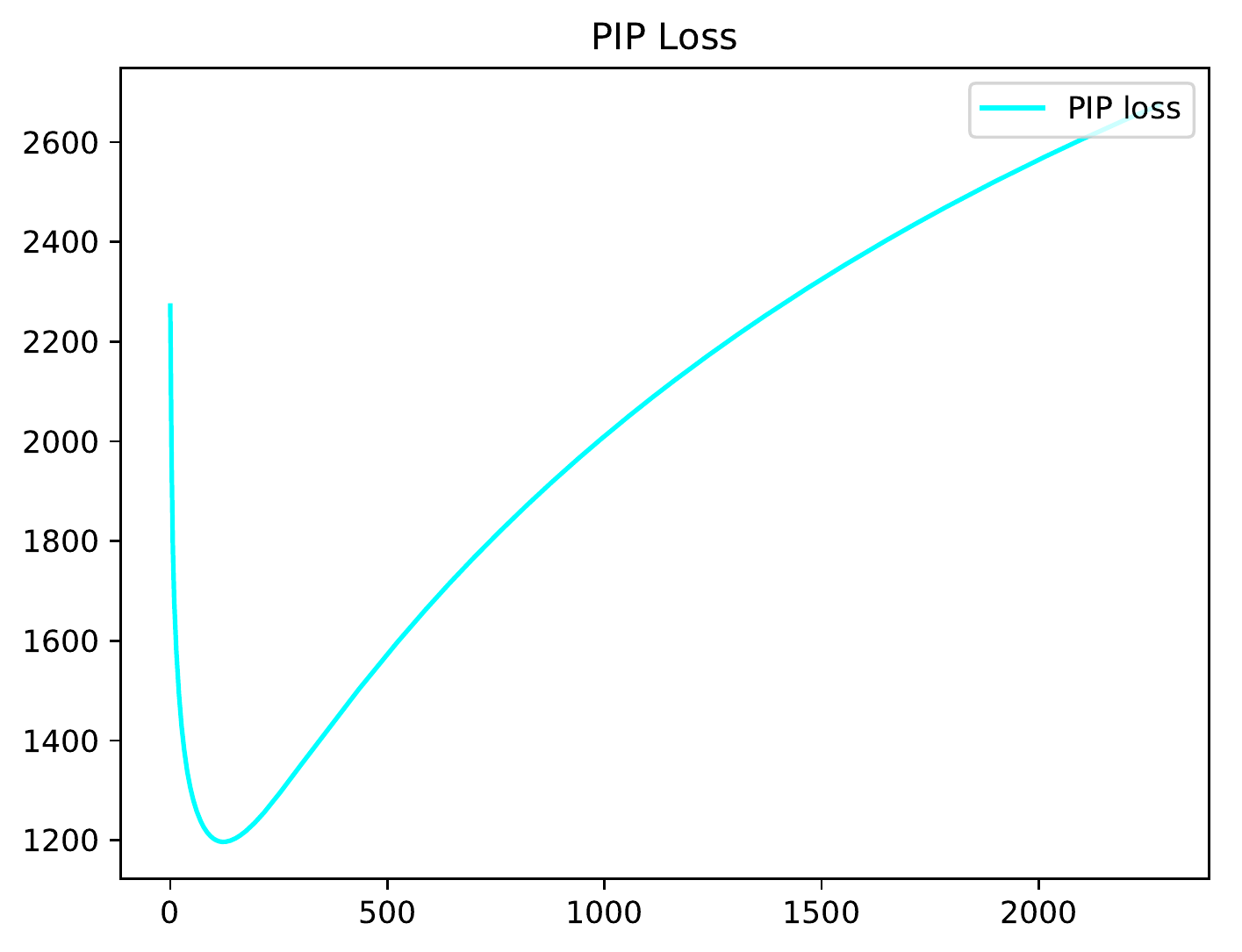}
\caption{PIP loss (in $y$-axis) against the dimensionality (in $x$-axis) for SGNS source embeddings computed with $\alpha=0.5$. The optimal dimensionality is 121.}
\label{fig:pip-sgns}
\end{figure}
\begin{figure}[t]
\centering
\includegraphics[width=8cm]{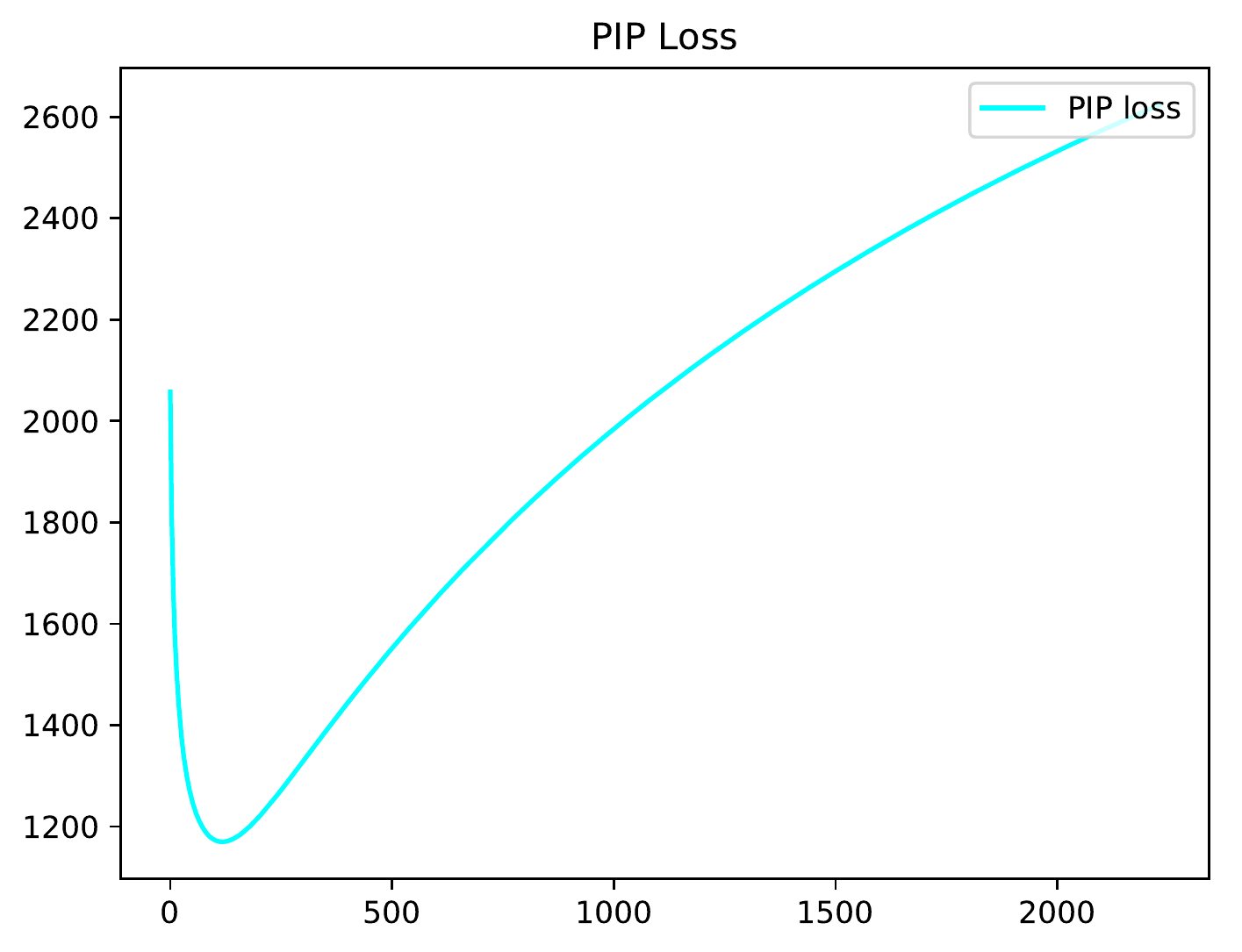}
\caption{PIP loss (in $y$-axis) against the dimensionality (in $x$-axis) for LSA source embeddings computed with $\alpha=0.5$. The optimal dimensionality is 119.}
\label{fig:pip-lsa}
\end{figure}

\section{Hyperparameters and Run time}
\label{sec:hyp}

\begin{table*}[t]
 \centering
 \begin{tabular}{l c c c c}\toprule
 Method & GloVe+LSA & GloVe+SGNS & LSA+SGNS & GloVe+LSA+SGNS \\ \midrule
 DAEME & 855 & 857 & 240 & 976 \\
 CAEME & 855 & 857 & 240 & 976 \\
 AAEME & 300 & 300 & 300 & 300 \\
 \bottomrule
 \end{tabular}
 \caption{Dimensionalities of the meta-embeddings created using autoencoded meta-embedding methods.}
 \label{tbl:AE-dim}
\end{table*}

All hyperparameters for the previously proposed meta-embedding learning methods were tuned on the MTurk-771 dataset~\cite{Halawi:2012}, such that the Pearson correlation coefficient measured between cosine similarities computed using the created meta-embeddings and human similarity ratings for the same word-pairs are maximised. 
For the autoencoded meta-embeddings~\cite{Bao:COLING:2018}, the dimensionality of the created meta-embeddings are shown in \autoref{tbl:AE-dim}.
Note that CAEME and DAEME produces meta-embeddings by concatenating source embeddings after applying a nonlinear encoder.
Therefore, the dimensionality of the meta-embedding is equal to the sum of input source embedding dimensionalities.
For AAEME, we vary the output dimensionality in the set ${50, 100, 200, 300, 400, 500, 600}$ and found that 300 to be the best.

For LLE~\cite{Bollegala:IJCAI:2018} we used a neighbourhood size of 2000 (found to be optimal among ${600, 1200, 2000}$) and created 200 dimensional embeddings (found to be optimal among ${100, 200, 300, 400, 450, 500, 600, 750}$), 
whereas for 1\texttt{TO}N~\cite{Yin:ACL:2016} the optimal dimensionality was 300 (found to be optimal among ${100, 200, 300, 400, 500, 600, 700, 800}$)
trained using stochastic gradient descent (learning rate of $0.1$ was found to be optimal among ${0.001, 0.01, 0.1, 1, 10, 100, 1000}$).
We used the publicly available implementations by the original authors for obtaining results for all previously proposed meta-embedding learning methods.

The average run times of both source-weighted and dimension-weighted meta-embedding learning algorithms is approximately 30 minutes (wall clock time),
measured on a EC2 p3.2xlarge instance with 1 Tesla V100 GPU, 16 GB GPU memory and 8 vCPUs.

\section{Visualisation of Embeddings}

\autoref{fig:tsne} shows t-SNE~\cite{tsne} visualisation for randomly selected positive and negative sentiment words from SentiWords~\cite{SentiWords} using DW meta-embeddings.
We see that overall positive (bottom) and negative (top) words are distributed in two groups, but at the same time semantically similar related words are projected close to each other.
A limitation of concatenation is the increase of dimensionality with the number of sources.
Learning lower dimensional unsupervised concatenated meta-embeddings is deferred to future work.

\begin{figure}
 \centering
 \includegraphics[width=78mm]{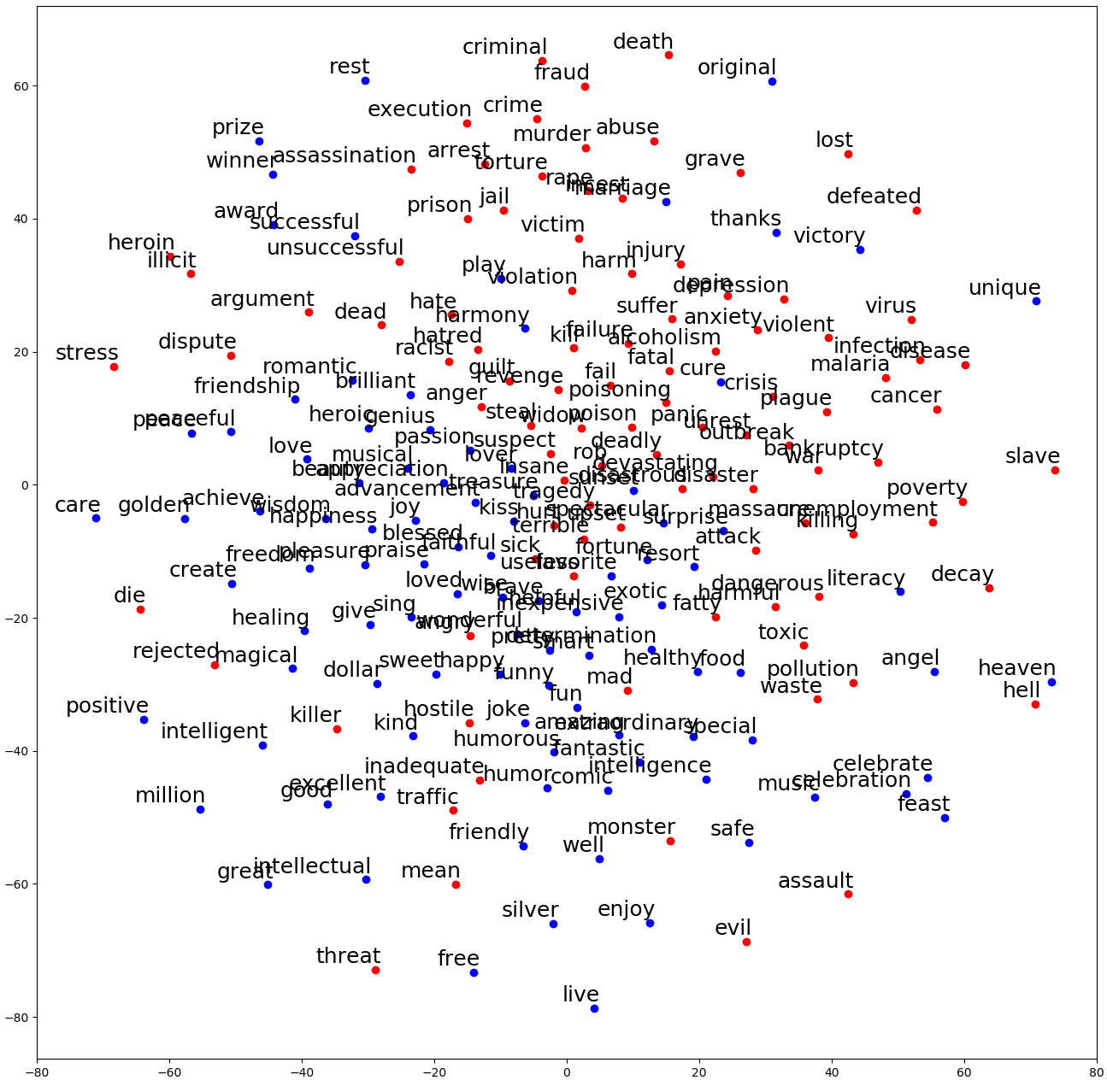}
 \caption{t-SNE visualisation of positive (blue) and negative (red) sentiment words using DW embeddings.}
 \label{fig:tsne}
\end{figure}

\section{Alternative Proof of Lemma 1}
\label{sec:alt-proof}

We first prove Lemma~\ref{lem:svd}.
\begin{lemma}
\label{lem:svd}
For orthogonal matrices $\mat{X}_{0} \in \R^{n \times k}$, $\mat{Y}_{1} \in \R^{n \times (n-k)}$, the SVD of their inner-product equals
\begin{align}
\label{eq:lem-svd}
\mathrm{SVD}(\mat{X}_{0}\T\mat{Y}_{1}) = \mat{U}_{0} \sin (\Theta){\tilde{\mat{V}}_{1}}\T
 \end{align}
 where $\Theta$ are the principal angles between $\mat{X}_{0}$ and $\mat{Y}_{0}$, the orthonormal complement of $\mat{Y}_{1}$.
\end{lemma}

\begin{proof}
Let us consider the eigendecomposition of $\mat{X}_{0}\T\mat{Y}_{1}(\mat{X}_{0}\T\mat{Y}_{1})\T$.
From the definition of orthonormal complement we have $\mat{Y}_{0}\mat{Y}_{0}\T + \mat{Y}_{1}\mat{Y}_{1}\T = \mat{I}$.
Therefore, we can substitute for $\mat{Y}_{1}\mat{Y}_{1}\T$ in the eigendecomposition to obtain:
\begin{align}
 \mat{X}_{0}\T\mat{Y}_{1}\mat{X}_{0}\T\mat{Y}_{1}\T &= \mat{X}_{0}\T(\mat{I} - \mat{Y}_{0}\mat{Y}_{0}\T)\mat{X}_{0} \\
 &= \mat{I} - \mat{U}_{0}\cos^{2}(\Theta)\mat{U}_{0}\T \label{eq:svd1} \\
 &= \mat{U}_{0}\sin^{2}(\Theta)\mat{U}_{0}\T \label{eq:svd2}
\end{align}
\eqref{eq:svd1} follows from the eigendecomposition of $\mat{Y}_{0}\mat{Y}_{0}\T$ where all eigenvalues are nonnegative and can be written as the square of some some cosines of angles, where norms are absorbed into $\mat{U}_{0}$.
Then, \eqref{eq:svd2} follows from the trigonometric identity, $\cos^{2}(\Theta) + \sin^{2}(\Theta) = 1$.
Therefore, $\mat{X}_{0}\T\mat{Y}_{1}$ has the singular value decomposition $\mat{U}_{0}\sin(\Theta){\tilde{\mat{V}}_{1}}\T$
\end{proof}

Now we prove Lemma~\ref{lem:decomp}.
\begin{proof}
Note $\mat{Y}_{0}=\mat{U}\mat{U}\T\mat{Y}_{0} = \mat{U} \left(
\begin{bmatrix}
    \mat{X}_{0}\T  \\
    \mat{X}_{1}\T
\end{bmatrix}\mat{Y}_{0} \right)$.
Therefore, we can write,
\begin{align*}
\mat{Y}_{0} \mat{Y}_{0}\T & = \mat{U} (\begin{bmatrix}
    \mat{X}_{0}\T  \\
    \mat{X}_{1}\T
\end{bmatrix} \mat{Y}_{0} \mat{Y}_{0}\T
\begin{bmatrix}
    \mat{X}_{0} & \mat{X}_{1}
\end{bmatrix})\mat{U}\T\\
& = \mat{U}
\begin{bmatrix}
    \mat{X}_{0}\T \mat{Y}_{0} \mat{Y}_{0}\T \mat{X}_{0} & \mat{X}_{0}\T \mat{Y}_{0} \mat{Y}_{0}\T \mat{X}_{1} \\
    \mat{X}_{1}\T \mat{Y}_{0} \mat{Y}_{0}\T \mat{X}_{0} & \mat{X}_{1}\T \mat{Y}_{0} \mat{Y}_{0}\T \mat{X}_{1}
\end{bmatrix}
\mat{U}\T
\end{align*}

Let $\mat{X}_{0}\T \mat{Y}_{0} = \mat{U}_{0}\cos(\Theta)\mat{V}_{0}\T$,  
$\mat{Y}_{0}\T \mat{X}_{1} = \mat{V}_{0} \sin(\Theta)\tilde{\mat{U}_{1}}\T$ by Lemma \ref{lem:svd}. For any unit invariant norm,
{\small
\begin{align}
&\norm{\mat{Y}_{0} \mat{Y}_{0}\T - \mat{X}_{0} \mat{X}_{0}\T}\\
=&\norm{\mat{U}
\left( \begin{bmatrix}
    \mat{X}_{0}\T \mat{Y}_{0} \mat{Y}_{0}\T \mat{X}_{0} & \mat{X}_{0}\T \mat{Y}_{0} \mat{Y}_{0}\T \mat{X}_{1} \\
    \mat{X}_{1}\T \mat{Y}_{0} \mat{Y}_{0}\T \mat{X}_{0} & \mat{X}_{1}\T \mat{Y}_{0} \mat{Y}_{0}\T \mat{X}_{1}
\end{bmatrix} -
\begin{bmatrix}
    \mat{I} & \mat{0} \\
    \mat{0} & \mat{0}
\end{bmatrix} \right)
\mat{U}\T} \label{eq:unit-1} \\
=&\norm{
\begin{bmatrix}
    \mat{X}_{0}\T \mat{Y}_{0} \mat{Y}_{0}\T \mat{X}_{0} & \mat{X}_{0}\T \mat{Y}_{0} \mat{Y}_{0}\T \mat{X}_{1} \\
    \mat{X}_{1}\T \mat{Y}_{0} \mat{Y}_{0}\T \mat{X}_{0} & \mat{X}_{1}\T \mat{Y}_{0} \mat{Y}_{0}\T \mat{X}_{1}
\end{bmatrix}-
\begin{bmatrix}
    \mat{I} & \mat{0} \\
    \mat{0} & \mat{0}
\end{bmatrix}
}\\
=&\norm{
\begin{bmatrix}
   \kern-5pt \mat{U}_{0}\cos^2(\Theta)\mat{U}_{0}\T & \kern-20pt \mat{U}_{0}\cos(\Theta)\sin(\Theta)\tilde{\mat{U}_{1}}\T \\
    \tilde{\mat{U}_{1}}\cos(\Theta)\sin(\Theta)\mat{U}_{0}\T & \kern-20pt \tilde{\mat{U}_{1}}\sin^2(\Theta)\tilde{\mat{U}_{1}}\T
\end{bmatrix}-
\begin{bmatrix}
    \mat{I} & \kern-5pt \mat{0} \\
    \mat{0} & \kern-5pt \mat{0}
\end{bmatrix}
}\\
=&\norm{
\begin{bmatrix}
    \mat{U}_{0} & \kern-10pt \mat{0} \\
    \mat{0} & \kern-10pt \tilde{\mat{U}_{1}}
\end{bmatrix}
\begin{bmatrix}
    -\sin^2(\Theta) & \kern-10pt \cos(\Theta)\sin(\Theta) \\
    \cos(\Theta)\sin(\Theta) & \kern-10pt \sin^2(\Theta)
\end{bmatrix}
\begin{bmatrix}
    \mat{U}_{0} & \mat{0} \\
    \mat{0} & \tilde{\mat{U}_{1}}
\end{bmatrix}^{\top}
}\\
=&\norm{
\begin{bmatrix}
    -\sin^2(\Theta) & \cos(\Theta)\sin(\Theta) \\
    \cos(\Theta)\sin(\Theta) & \sin^2(\Theta)
\end{bmatrix}
}\\
=&\norm{
\begin{bmatrix}
\sin(\Theta) & 0\\
0 & \sin(\Theta)
\end{bmatrix}
\begin{bmatrix}
    -\sin(\Theta) & \cos(\Theta) \\
    \cos(\Theta) & \sin(\Theta)
\end{bmatrix}
}\label{eq:unit-2}\\
=&\norm{
\begin{bmatrix}
\sin(\Theta) & 0\\
0 & \sin(\Theta)
\end{bmatrix}}
\end{align}
}%

In Lines~\ref{eq:unit-1} and \ref{eq:unit-2} we have used the fact that the norm is invariant under the multiplication of an orthogonal matrix.

On the other hand by the definition of principal angles we have
$\norm{\mat{X}_{0}\T\mat{Y}_{1}} = \norm{\sin(\Theta)}$. 
So we established the lemma. Specifically, we have
\[ \norm{\mat{X}_{0} \mat{X}_{0}\T - \mat{Y}_{0}\mat{Y}_{0}\T}_F = \sqrt{2}\norm{\mat{X}_{0}\T\mat{Y}_{1}}_F \]
\end{proof}
Without loss of soundness, we omitted in the proof sub-blocks of identities or zeros for simplicity. Interested readers can refer to classical matrix decomposition texts~\cite{Stewart:1990,Kato:2013} for a comprehensive treatment of this topic.

\bibliography{Embed}
\bibliographystyle{named}

\end{document}